\documentclass{article}

\usepackage[utf8]{inputenc} 
\usepackage[T1]{fontenc}    
\usepackage[round]{natbib}
\usepackage[english]{babel}
\usepackage{hyperref}
\usepackage{amsmath,amsthm,amssymb}
\usepackage{subcaption}
\usepackage{mdframed} 

\usepackage{algorithm}
\usepackage{algorithmic}
\usepackage[titlenumbered,ruled,noend,algo2e]{algorithm2e}

\SetCommentSty{mycommfont}
\SetEndCharOfAlgoLine{}

\usepackage[accepted]{icml2020}
\usepackage{xcolor}
\definecolor{linkcolor}{RGB}{83,83,182}
\definecolor{citecolor}{RGB}{128,0,128}
\hypersetup{
    colorlinks=true,
    citecolor=citecolor,
    linkcolor=linkcolor,
    urlcolor=linkcolor
}

\usepackage[capitalise,nameinlink]{cleveref}
\usepackage{booktabs}
\usepackage{quentin_shortcuts}

\usepackage{chngcntr}
\usepackage{apptools}
\usepackage{siunitx}
\AtAppendix{\counterwithin{lemma}{section}}
\AtAppendix{\counterwithin{lemmaenumi}{section}}
\AtAppendix{\counterwithin{theorem}{section}}

\begin{document}
\twocolumn[
\icmltitle{Implicit differentiation of Lasso-type models for hyperparameter optimization}



\icmlsetsymbol{equal}{*}

\icmlsetsymbol{equal}{*}

\begin{icmlauthorlist}
\icmlauthor{Quentin Bertrand}{equal,inria}
\icmlauthor{Quentin Klopfenstein}{equal,imb}
\icmlauthor{Mathieu Blondel}{google}
\icmlauthor{Samuel Vaiter}{cnrsimb}
\icmlauthor{Alexandre Gramfort}{inria}
\icmlauthor{Joseph Salmon}{mtp}
\end{icmlauthorlist}

\icmlaffiliation{inria}{Universit\'e Paris-Saclay, Inria, CEA, Palaiseau, France}
\icmlaffiliation{imb}{Institut  Math\'ematique  de  Bourgogne,   Universit\'e  de  Bourgogne, Dijon, France}
\icmlaffiliation{cnrsimb}{CNRS and Institut  Math\'ematique  de  Bourgogne,   Universit\'e  de  Bourgogne, Dijon, France}
\icmlaffiliation{mtp}{IMAG, Universit\'e de Montpellier, CNRS, Montpellier, France}
\icmlaffiliation{google}{Google Research, Brain team, Paris, France}

\icmlcorrespondingauthor{Quentin Bertrand}{quentin.bertrand@inria.fr}
\icmlcorrespondingauthor{Quentin Klopfenstein}{quentin.klopfenstein@u-bourgogne.fr}

\icmlkeywords{Machine Learning, ICML}

\vskip 0.3in
]



\printAffiliationsAndNotice{\icmlEqualContribution} 

\begin{abstract}
Setting regularization parameters for Lasso-type estimators is notoriously difficult, though crucial in practice.
The most popular hyperparameter optimization approach is grid-search using held-out validation data.
Grid-search however requires to choose a predefined grid for each parameter, which scales exponentially in the number of parameters.
Another approach is to cast hyperparameter optimization as a bi-level optimization problem,
one can solve by gradient descent.
The key challenge for these methods is the estimation of the gradient \wrt the hyperparameters.
Computing this gradient via forward or backward automatic differentiation is possible yet usually suffers from high memory consumption.
Alternatively implicit differentiation typically involves solving a linear system which can be prohibitive and numerically unstable in high dimension.
In addition, implicit differentiation usually assumes smooth loss functions, which is not the case for Lasso-type problems.
This work introduces an efficient implicit differentiation
algorithm, \emph{without} matrix inversion, tailored for Lasso-type problems.
Our approach scales to high-dimensional data by leveraging the sparsity of the solutions.
Experiments demonstrate that the proposed method outperforms a large number of standard methods to
optimize the error on held-out data, or the Stein Unbiased Risk Estimator (SURE).

\end{abstract}

\section{Introduction}

In many statistical applications, the number of parameters $p$ is much larger than the number of observations $n$.
In such scenarios, a popular approach to tackle linear regression problems is to consider convex $\ell_1$-type penalties, used in Lasso~\citep{Tibshirani96}, Group-Lasso~\citep{Yuan_Lin06}, Elastic-Net~\citep{Zou_Hastie05} or adaptive Lasso~\citep{Zou06}.
These \emph{Lasso-type} estimators rely on regularization hyperparameters,
trading data fidelity against sparsity.
Unfortunately, setting these hyperparameters is hard in practice: estimators based on $\ell_1$-type penalties are indeed more sensitive to the choice of hyperparameters than $\ell_2$ regularized estimators.

To control for overfitting, it is customary to use different datasets for model training (\ie computing the regression coefficients) and hyperparameter selection (\ie choosing the best regularization parameters).
A \textit{metric}, \eg \emph{hold-out loss}, is optimized on a validation dataset~\citep{Stone_Ramer65}.
Alternatively one can rely on a statistical criteria that penalizes complex models such as AIC/BIC \citep{Liu_Yang11} or SURE (Stein Unbiased Risk Estimator, \citealt{Stein81}).
In all cases, hyperparameters are tuned to optimize a chosen metric.

The canonical hyperparameter optimization method is \emph{grid-search}.
It consists in fitting and selecting the best model over a predefined grid of parameter values.
The complexity of grid-search is exponential with the number of
hyperparameters,
making it only competitive when the number of hyperparameters is small.
Other hyperparameter selection strategies include \emph{random search} \citep{Bergstra_Bengio12} and Bayesian optimization \citep{Brochu_Cora_deFreitas10,Snoek_Larochelle_Ryan12} that aims to learn an approximation of the metric over the parameter space and rely on an exploration policy to find the optimum.

Another line of work for hyperparameter optimization (HO) relies on gradient
descent in the hyperparameter space.
This strategy has been widely explored for smooth objective functions
\citep{Larsen_Hansen_Svarer_Ohlsson96,Bengio00,Larsen_Svarer_Andersen_Hansen12}.
The main challenge for this class of methods is
estimating the gradient \wrt the hyperparameters.
Gradient estimation techniques are mostly divided in two categories.
\emph{Implicit differentiation}
requires the exact solution of the optimization problem and involves the resolution of a linear system~\cite{Bengio00}.
This can be expensive to compute and lead to numerical instabilities, especially when the system is ill-conditioned~\cite{Lorraine_Vicol_Duvenaud2019}.
Alternatively, \textit{iterative differentiation} computes the
gradient using the iterates of an optimization algorithm.
Backward iterative differentiation~\citep{Domke12} is computationally efficient when the
number of hyperparameters is large.
However it is memory
consuming since it requires storing all intermediate iterates.
In contrast, forward iterative differentiation
\citep{Deledalle_Vaiter_Fadili_Peyre14,Franceschi_Donini_Frasconi_Pontil17} does
not require storing the iterates but can be computationally expensive with a
large number of hyperparameters;
see \citet{Baydin_Pearlmutter_Radul_Siskind18} for a survey.

This article proposes to investigate the use of these methods to set the regularization hyperparameters in an automatic fashion for
Lasso-type problems.
To cover the cases of both low and high number of hyperparameters, two estimators are investigated, namely the Lasso and the weighted Lasso which have respectively one or as many parameters as features.
Our contributions are as follows:
\begin{itemize}
    \item We show that forward iterative differentiation of
    block coordinate descent (BCD), a state-of-the-art solver for Lasso-type
    problems, converges towards the true gradient.
    Crucially, we show that this scheme converges linearly once the support is
    identified and that its limit does \textbf{not} depend of the initial starting point.
    \item These results lead to the proposed algorithm (\Cref{alg:compute_jac_imp_forward_iter_diff}) where the computation of the Jacobian is \textbf{decoupled} from the computation of the regression coefficients.
    The later can be done with state-of-the-art convex solvers, and interestingly, it does not require solving a linear system, potentially ill-conditioned.
    \item We show through an extensive benchmark on simulated and real high dimensional data that the proposed method outperforms state-of-the-art HO methods.
\end{itemize}
Our work is somewhat similar to \citet{Gregor_LeCun10,Xin_Wang_Gao_Wipf_Wang16,Borgerding17,Liu18,Wu_Guo_Li_Zhang19}, where the \emph{solution} is differentiated w.r.t. optimization parameters instead of the regularization parameter.
However the goal is very different as they want to accelerate the optimization algorithm whereas we provide an efficient algorithm to compute the gradient.

\textbf{Notation}
The design matrix is $X \in \bbR^{n \times p}$ (corresponding to $n$ samples and $p$ features) and the observation vector is $y \in \bbR^n$.
The regularization parameter, possibly multivariate, is denoted by $\lambda = (\lambda_1,\dots,\lambda_r)^\top\in\bbR^r$.
We denote $\hat \beta^{(\lambda)} \in \bbR^p$ the regression coefficients associated to $\lambda$.
We denote $\hat \jac_{(\lambda)} \eqdef (\nabla_\lambda \hat \beta_1^{(\lambda)}, \dots, \nabla_\lambda \hat\beta_p^{(\lambda)})^\top \in \bbR^{p \times r}$ the weak Jacobian \citep{evan1992measure} of $\hat \beta^{(\lambda)}$ \wrt $\lambda$.
For a function $\psi: \bbR^{p} \times \bbR^{r} \to \bbR$ with weak derivatives of order two, we denote by $\nabla_\beta \psi(\beta, \lambda) \in \bbR^p$ (resp. $\nabla_\lambda (\beta, \lambda) \in \bbR^r$) its weak gradient \wrt the first parameter (resp. the second parameter). The weak Hessian $\nabla^2 \psi(\beta, \lambda)$ is a matrix in $\bbR^{(p+r) \times (p+r)}$ which has a block structure
\begin{equation*}
    \nabla^2 \psi(\beta, \lambda) =
    \begin{pmatrix}
        \nabla_\beta^2 \psi(\beta, \lambda) & \nabla_{\beta,\lambda}^2 \psi(\beta, \lambda) \\
        \nabla_{\lambda,\beta}^2 \psi(\beta, \lambda) & \nabla_{\lambda}^2 \psi(\beta, \lambda)
    \end{pmatrix} \enspace .
\end{equation*}
The support of $\hat\beta^{(\lambda)}$ (the indices of non-zero coefficients) is denoted by $\hat S^{(\lambda)}$, and $\hat s^{(\lambda)}$ represents its cardinality (\ie the number of non-zero coefficients).
The sign vector $\sign \hat\beta^{(\lambda)} \in \bbR^p$ is the vector of component-wise signs (with the convention that $\sign(0) = 0$) of $\hat\beta^{(\lambda)}$.
Note that to ease the reading, we drop $\lambda$ in the notation when it is clear from
the context and use $\hat \beta, \hat \jac, \hat S$ and $\hat{s}$.
The Mahalanobis distance of a vector $x \in \bbR^p$ and a matrix $A \succ 0$ is noted $\norm{x}_A \eqdef \sqrt{x^\top A^{-1} x}$.

\section{Background}

\subsection{Problem setting}
%
To favor sparse coefficients, we consider Lasso-type estimators based on non-smooth regularization functions. Such problems consist in finding:
\begin{problem}\label{pb:lasso_typ}
    \hat \beta^{(\lambda)}
    \in
    \argmin_{\beta \in \bbR^p} \psi(\beta,\lambda) \enspace.
\end{problem}
The Lasso \cite{Tibshirani96} is recovered, with the number of hyperparameters
set to $r=1$:
\begin{problem}\label{pb:lasso}
    \psi(\beta,\lambda) =
    \frac{1}{2n} \normin{y - X \beta}_2^2+ e^\lambda \normin{\beta}_1 \enspace,
\end{problem}
while the weighted Lasso (wLasso,
\citealt{Zou06}, introduced to reduce the bias of the Lasso) has $r=p$ hyperparameters and reads:
\begin{problem}\label{pb:alasso}
    \psi(\beta,\lambda) =
    \frac{1}{2n} \normin{y - X \beta}_2^2
    + \sum_{j=1}^p e^{\lambda_j} |\beta_j| \enspace.
\end{problem}
Note that we adopt the hyperparameter parametrization of \citet{Pedregosa16}, \ie we write the regularization parameter as $e^\lambda$.
This avoids working with a positivity constraint in the optimization process and fixes scaling issues in the line search. It is also coherent with the usual choice of a geometric grid for grid-search~\citep{Friedman_Hastie_Tibshirani10}.

\begin{remark}
Other formulations could be investigated
like Elastic-Net or non-convex formulation, \eg MCP \cite{Zhang10}.
Our theory does not cover non-convex cases, though we illustrate that it  behaves properly numerically. Handling such non-convex cases is left as a question for future work.
\end{remark}
The HO problem can be expressed as a nested \textit{bi-level optimization} problem. For a given differentiable criterion $\mathcal{C}: \bbR^p \mapsto \bbR$ (\eg hold-out loss or SURE), it reads:

\begin{align}\label[pb_multiline]{eq:bilevel_opt}
    \argmin_{\lambda \in \bbR^r}
    &
    \left\{
    \mathcal{L}(\lambda) \eqdef
    \mathcal{C} \left (\hat \beta^{(\lambda)}  \right)
    \right\}
    \nonumber \\
    &\st \hat \beta^{(\lambda)} \in \argmin_{\beta \in \bbR^p}
      {
      \psi(\beta,\lambda)
      }
     \enspace.
\end{align}
Note that SURE itself is not necessarily weakly differentiable \wrt $\hat
\beta^{(\lambda)}$. However a weakly differentiable approximation can be
constructed \citep{Ramani_Blu_Unser08,
Deledalle_Vaiter_Fadili_Peyre14}.
Under the hypothesis that \Cref{pb:lasso_typ} has a unique solution for every $\lambda \in \bbR^r$, the function $\lambda \mapsto  \hat \beta^{(\lambda) }$ is weakly differentiable \citep{Vaiter_Deledalle_Peyre_Dossal_Fadili13}.
Using the chain rule, the gradient of $\mathcal{L}$ \wrt $\lambda$ then writes:
\begin{align}\label{eq:grad_crit}
    \nabla_{\lambda} \mathcal{L}(\lambda)
    = \hat \jac_{(\lambda)}^\top \nabla \mathcal{C} \left ( \hat \beta^{(\lambda)} \right ) \enspace .
\end{align}
Computing the weak Jacobian $ \hat \jac_{(\lambda)}$ of the inner problem is the main challenge, as once the \emph{hypergradient}
$\nabla_{\lambda} \mathcal{L}(\lambda)$ has been computed, one can use usual
gradient descent, $\lambda^{(t+1)} = \lambda^{(t)} - \rho\nabla_{\lambda}
\cL(\lambda^{(t)}),$ for a step size $\rho > 0$.
Note however that $\mathcal{L}$ is usually non-convex and convergence towards a global minimum is not guaranteed.
In this work, we propose an efficient algorithm to compute $\hat
\jac_{(\lambda)}$ for Lasso-type problems, relying on improved forward differentiation.
%
\subsection{Implicit differentiation (smooth case)}
%
Implicit differentiation, which can be traced back to
\citet{Larsen_Hansen_Svarer_Ohlsson96}, is based on the knowledge of
$\hat{\beta}$ and requires solving a $p\times p$ linear system \citep[Sec.\  4]{Bengio00}.
Since then, it has been extensively applied in various contexts.
\citet{Chapelle_Vapnick_Bousquet_Mukherjee02,Seeger08} used implicit differentiation to select hyperparameters of kernel-based models.
\citet{Kunisch_Pock13} applied it to image restoration.
\citet{Pedregosa16} showed that each inner optimization problem could be solved only approximately, leveraging noisy gradients.
Related to our work, \citet{Foo_Do_Ng08} applied implicit differentiation on
a ``weighted''
Ridge-type estimator (\ie a Ridge penalty with one $\lambda_j$ per feature).

Yet, all the aforementioned methods have a common drawback : they are limited to the smooth setting, since they rely on optimality conditions for smooth optimization.
They proceed as follows: if $\beta \mapsto \psi(\beta,\lambda)$ is a smooth convex function (for any fixed $\lambda$) in \Cref{pb:lasso_typ}, then for all $\lambda$, the solution $\hat \beta^{(\lambda)}$ satisfies the following fixed point equation:
\begin{align} \label{eq:fixed_point_eq}
   \nabla_\beta \psi \left( \hat \beta^{(\lambda)}, \lambda \right )
   = 0
   \enspace .
\end{align}
Then, this equation can be differentiated \wrt $\lambda$:
\begin{align}
    \nabla_{\beta, \lambda}^2 \psi(\hat \beta^{(\lambda)}, \lambda)
    + \hat \jac_{(\lambda)}^\top \nabla_{\beta}^2
    \psi(\hat \beta^{(\lambda)}, \lambda)
    = 0 \enspace .
\end{align}
Assuming that $\nabla_{\beta}^2 \psi( \hat \beta^{(\lambda)}, \lambda)$ is
invertible this leads to a closed form  solution
for the weak Jacobian $\hat \jac_{(\lambda)}$:
\begin{align}
    \hat \jac_{(\lambda)}^\top
    = - \nabla_{\beta, \lambda}^2
    \psi \left(\hat \beta^{(\lambda)}, \lambda \right)
    {\underbrace{\left ( \nabla_{\beta}^2 \psi(\beta^{(\lambda)}, \lambda) \right )}_{p \times p}}^{-1}
    \enspace,
\end{align}
which in practice is computed by solving a linear system.
Unfortunately this approach cannot be generalized for non-smooth problems since \Cref{eq:fixed_point_eq} no longer holds.

\subsection{Implicit differentiation (non-smooth case)}

Related to our work \citet{Mairal_Bach_Ponce12} used implicit differentiation with respect to the dictionary ($X \in \bbR^{n \times p}$) on Elastic-Net models to perform dictionary learning.
Regarding Lasso problems, the literature is quite scarce, see~\citep{Dossal_Kachour_Fadili_Peyre_Chesneau12,Zou_Hastie_Tibshirani07} and~\citep{Vaiter_Deledalle_Peyre_Dossal_Fadili13,Tibshirani_Taylor11} for a more generic setting encompassing weighted Lasso.
General methods for gradient estimation of non-smooth optimization schemes exist~\citep{Vaiter_Deledalle_Peyre_Fadili_Dossal17} but are not practical since they depend on a possibly ill-posed linear system to invert.
\citet{Amos_Brandon17} have applied implicit differentiation on estimators based on quadratic objective function with linear constraints, whereas \citet{Niculae_Blondel17} have used implicit differentiation on a smooth objective function with simplex constraints.
However none of these approaches leverages the sparsity of Lasso-type estimators.

%
\section{Hypergradients for Lasso-type problems}
%
To tackle hyperparameter optimization of non-smooth Lasso-type problems,
we propose in this section an efficient algorithm for hypergradient estimation.
Our algorithm relies on implicit differentiation,
thus enjoying low-memory cost, yet does not require to
naively solve a (potentially ill-conditioned) linear system of equations. In the sequel, we
assume access to a (weighted) Lasso solver, such as
ISTA~\citep{Daubechies_Defrise_DeMol04} or Block Coordinate Descent (BCD,
\citealt{Tseng_Yun09}, see also \cref{alg:bcd_lasso}).

\subsection{Implicit differentiation}
\label{sub:implicit_diff_non_smooth}

Our starting point is the key observation that Lasso-type solvers induce a fixed
point iteration that we can leverage to compute a Jacobian.
Indeed, proximal BCD algorithms \citep{Tseng_Yun09}, consist in a local
gradient step composed with a soft-thresholding step (ST), \eg for the Lasso, for $j \in 1, \dots, p$:
\begin{align}\label{eq:coordinate_descent_lasso}
    \beta_{j} \leftarrow
    \ST \left (
        \beta_j
        - \frac{ X_{:, j}^\top (X \beta - y)}{\norm{X_{:, j}}^2},
         \frac{ne^\lambda}{\norm{X_{:, j}}^2 } \right )
    \enspace
\end{align}
where $\ST(t,\tau) = \sign(t) \cdot(|t|-\tau)_{+}$ for any $t\in\bbR$ and $\tau\geq0$ (extended for vectors component-wise).
The solution of the optimization problem satisfies, for any $\alpha>0$, the fixed-point equation \citep[Prop. 3.1]{Combettes_Wajs05}, for $j \in 1, \dots, p$:
\begin{align}\label{eq:fixed_point_lasso}
    \hbeta_j^{(\lambda)} =
    \ST \left (
        \hbeta_j^{(\lambda)} - \frac{1}{\alpha} X_{j, :}^\top (X \hbeta^{(\lambda)} - y),
         \frac{ne^\lambda}{\alpha} \right )
    \enspace .
\end{align}
The former can be differentiated \wrt $\lambda$, see \cref{app:lemma_st_deriv} in Appendix, leading to a closed form solution for the Jacobian $\jac_{(\lambda)}$ of the Lasso and the weighted Lasso.
\setcounter{proposition}{0}
\begin{mdframed}[linewidth=0.5pt]
\begin{proposition}[Adapting {\citealt[Thm. 1]{Vaiter_Deledalle_Peyre_Dossal_Fadili13}}]\label{prop:closed_form_jac_lasso}
Let $\hat S$ be the support of the vector $\hbeta^{(\lambda)}$.
Suppose that $X_{\hat S}^\top X_{\hat S} \succ 0$
, then a weak Jacobian $\hat\jac=\hat\jac_{(\lambda)}$ of the Lasso writes:
\begin{align}\label{eq:closed_form_lasso}
    \hat \jac_{\hat S} & =
    - n e^\lambda \left (X_{\hat S}^\top X_{\hat S} \right )^{-1}
    \sign \hbeta_{\hat S},\\
    \hat \jac_{{\hat S}^c} & = 0 \enspace,
\end{align}
and for the weighted Lasso:
\begin{flalign}
    \hat \jac_{\hat S, \hat S}
    &= -\! \left(
            X_{\hat S}^\top X_{\hat S}
        \right)^{\!-1}
    \!\diag \big (
            n e^{\lambda_{\hat S} } \odot \sign \hbeta_{\hat S}
        \big)\\
    \hat \jac_{j_1, j_2}
    &= 0 \quad \text{ if } j_1 \notin \hat S \text{ or }\text{ if } j_2 \notin \hat S \enspace.
\end{flalign}
\end{proposition}
\end{mdframed}
The proof of \Cref{prop:closed_form_jac_lasso} can be found in \Cref{sub:proof_of_prop:closed_form_jac_lasso}.
Note that the positivity condition in \Cref{prop:closed_form_jac_lasso} is satisfied if the (weighted) Lasso has a unique solution.
Moreover, even for multiple solutions cases, there exists at least one satisfying the positivity condition~\citep{Vaiter_Deledalle_Peyre_Dossal_Fadili13}.

\Cref{prop:closed_form_jac_lasso} shows that the Jacobian of the weighted Lasso $\hat \jac_{(\lambda)} \in \bbR^{p \times p}$ is row and column sparse.
This is key for algorithmic efficiency.
Indeed, \emph{a priori}, one
has to store a possibly dense
$p \times p$ matrix, which is prohibitive when $p$ is large.
\Cref{prop:closed_form_jac_lasso} leads to a simple algorithm (see \Cref{alg:compute_jac_implicit_diff})
to compute the Jacobian in a \emph{cheap} way, as it \emph{only} requires storing and inverting an $\hat s \times \hat s$ matrix.
Even if the linear system to solve is of size $\hat s \times \hat s$, instead of $p \times p$ for smooth objective function, the system to invert can be ill-conditioned, especially when a large support size $\hat s$ is encountered.
This leads to numerical instabilities
and slows down the resolution (see an illustration in \Cref{fig:Lasso_train_test_perf}).
Forward (\Cref{alg:compute_jac_forward_diff_bcd} in Appendix) and backward  (\Cref{alg:compute_jac_backward_iter_diff} in Appendix) iterative differentiation, which do not require solving linear systems, can overcome these issues.

%
\subsection{Link with iterative differentiation}
%
Iterative differentiation in the field of hyperparameter setting can be traced back to \citet{Domke12} who derived a backward differentiation algorithm for gradient descent, heavy ball and L-BFGS algorithms applied to smooth loss functions.
\citet{Agrawal_Amos_Barratt_Boyd_Diamond_Kolter19} generalized it to a specific subset of convex programs.
\citet{MacLaurin_Duvenaud_Adams15} derived a backward differentiation for stochastic gradient descent.
On the other hand \citet{Deledalle_Vaiter_Fadili_Peyre14} used forward differentiation of (accelerated) proximal gradient descent for hyperparameter optimization with non-smooth penalties.
\citet{Franceschi_Donini_Frasconi_Pontil17} proposed a benchmark of forward mode versus backward mode, varying the number of hyperparameters to learn.
\citet{Frecon_Salzo_Pontil2018} cast the problem of inferring the groups in a group-Lasso model as a bi-level optimization problem and solved it using backward differentiation.

Forward differentiation consists in differentiating each step of the algorithm (\wrt $\lambda$ in our case).
For the Lasso solved with BCD it amounts differentiating \Cref{eq:coordinate_descent_lasso}, and leads to the following recursive equation for the Jacobian, for $j \in 1, \dots p$, with $z_j= \beta_j
-  X_{:, j}^\top (X \beta - y) / \norm{X_{:, j}}^2  $:
\begin{align}
    \jac_j \leftarrow
    & \partial_1
    \ST \left ( z_j,\tfrac{ne^\lambda}{\norm{X_{:, j}}^2}\right )
        \left(
            \jac_j - \tfrac{1}{\norm{X_{:, j}}^2}X_{:, j}^\top X \jac
        \right)
        \nonumber
    \\
    & + \partial_2 \ST
    \left ( z_j, \tfrac{ne^\lambda}{\norm{X_{:, j}}^2} \right ) \tfrac{ne^\lambda}{\norm{X_{:, j}}^2}\enspace,
\end{align}
see \Cref{alg:compute_jac_forward_diff_bcd} (in Appendix) for full details.
Our proposed algorithm uses the fact that after a finite number of epochs $\partial_1 \ST(z_j, ne^\lambda / \norm{X_{:, j}}^2)$ and $\partial_2 \ST(z_j, ne^\lambda / \norm{X_{:, j}}^2)$ are \textbf{constant} (they no  no longer depends on the current $\beta$). Indeed, the sign of $\hat \beta$ is identified after a finite number of iterations thus the partial derivatives are constant.
 It is then possible to \textbf{decouple} the computation of the Jacobian by only solving \Cref{pb:lasso_typ} in a first step
 and then apply the forward
differentiation recursion steps, see \Cref{alg:compute_jac_imp_forward_iter_diff}.
This can be seen as the forward counterpart in a non-smooth case of the recent paper \citet{Lorraine_Vicol_Duvenaud2019}.
%
%
{\fontsize{4}{4}\selectfont
\begin{algorithm}[t]
\SetKwInOut{Input}{input}
\SetKwInOut{Init}{init}
\SetKwInOut{Parameter}{param}
\caption{\textsc{Implicit differentiation}
}
\Input{$
    X \in \bbR^{n \times p},
    y \in \bbR^{n},
    \lambda \in \bbR,
    n_{\text{iter}} \in \bbN$}
    \tcp{jointly compute coef. and Jacobian}

    \If{Lasso}{
    Get $\hbeta = Lasso(X, y, \lambda, n_{\text{iter}})
    $ and its support $\hat S$.

    $\hat \jac_{\phantom{\hat S}} = 0_{p}$

    $\hat \jac_{\hat S} =
    - n e^\lambda (X_{\hat S}^\top X_{\hat S})^{-1} \sign \hbeta_{\hat S} $
    }
    \If{wLasso}{
    Get $\hbeta = wLasso(X, y, \lambda, n_{\text{iter}})
    $ and its support $\hat S$.

    $\hat \jac = 0_{p \times p}$

    $\hat \jac_{\hat S, \hat S} =
    - (X_{\hat S}^\top X_{\hat S})^{-1}
    \diag ( n e^{\lambda_{\hat S}}
    \odot \sign \hbeta_{\hat S})$
    }
\Return{$\hbeta, \hat \jac$}
\label{alg:compute_jac_implicit_diff}
\end{algorithm}
}
An additional benefit of such updates is that they can be restricted to the (current) support, which leads to faster Jacobian computation.

We now show that the Jacobian computed using forward differentiation and our
method, \Cref{alg:compute_jac_imp_forward_iter_diff}, converges toward the true
Jacobian.
\begin{mdframed}[linewidth=0.5pt]
\begin{proposition}\label{prop:convergence_iterdiff}
    Assuming the Lasso solution (\Cref{pb:lasso}) (or weighted Lasso \Cref{pb:alasso}) is unique, then \Cref{alg:compute_jac_imp_forward_iter_diff,alg:compute_jac_forward_diff_bcd} converge toward
    the Jacobian
    $\hat \jac$ defined  in \Cref{prop:closed_form_jac_lasso}.
    \Cref{alg:compute_jac_forward_diff_bcd} computes the Jacobian along with the regression coefficients, once the support has been identified, the Jacobian converges linearly.
    \Cref{alg:compute_jac_imp_forward_iter_diff} computes first the coefficients $\hat \beta$ and then the Jacobian $\hat \jac$, provided that the support has been identified in the first step, the convergence is linear in the second, with the same rate as \Cref{alg:compute_jac_forward_diff_bcd}:
    \begin{equation}
        \normin{\jac_{\hat S}^{(k+1)} - \hat \jac}_{(X_{:, \hat S}^{\top} X_{:, \hat S})^{-1}}
        \leq
        C^k
        \normin{\jac_{\hat S}^{(k)} - \hat \jac}_{(X_{:, \hat S}^{\top} X_{:, \hat S})^{-1}}
        \nonumber
    \end{equation}
    where
    $C = \normin{A^{(j_{\hat s})} \dots A^{(j_1)}}_2 <1$, $j_1, \dots, j_{\hat s}$
    are the indices of the support of  $\hat \beta$ in increasing order
    and
     \begin{equation}
        A^{(j_s)} =
            \Id_{\hat s}
            - \frac{\left ( X_{:, \hat S}^\top X_{:, \hat S} \right )^{1/2}_{:, j_s} }{\norm{X_{:, j_s}}}
             \frac{\left ( X_{:, \hat S}^\top X_{:, \hat S} \right )^{1/2}_{j_s, :} }{\norm{X_{:, j_s}}} \in \bbR^{\hat s \times \hat s}
             .
        \nonumber
    \end{equation}
\qed
\end{proposition}
\end{mdframed}
Proof of \Cref{prop:convergence_iterdiff} can be found in \Cref{app:sub_conv_jac_ista,app:sub_conv_jac_bcd}.

\begin{remark}
    \emph{Uniqueness.}
    As proved in \citet[Lem. 3 and 4]{Tibshirani13} the set of (pathological) lambdas where the Lasso solution is not unique is typically empty.
    Moreover if the Lasso solution is not unique, there could be a non-continuous solution path
    $\lambda \mapsto \hat \beta^{(\lambda)} $,
    leaving only non-gradient based methods available.
    Even if \Cref{prop:convergence_iterdiff} does not provide theoretical guarantees in such a pathological setting, one can still apply \Cref{alg:compute_jac_imp_forward_iter_diff,alg:compute_jac_forward_diff_bcd},
    see \Cref{app:non_unique} for experiments in this settings.
\end{remark}
\begin{remark}
    \emph{Rate for the backward differentiation.}
    The backward and forward differentiation compute the same quantity: $\nabla_\lambda \cL(\lambda)$, but the backward differentiation directly computes the product given in \Cref{eq:grad_crit} leading to the gradient of $\mathcal{L}(\lambda)$. \Cref{prop:convergence_iterdiff} provides rates for the convergence of the Jacobian $\jac$ which leads to rates for the gradient \ie for the backward algorithm as well.
\end{remark}
\def \figsize {1}
\def \figprop {2}
\def \colprop {0.5}
%
\begin{figure}[tb]
    \includegraphics[width=\figsize\linewidth]{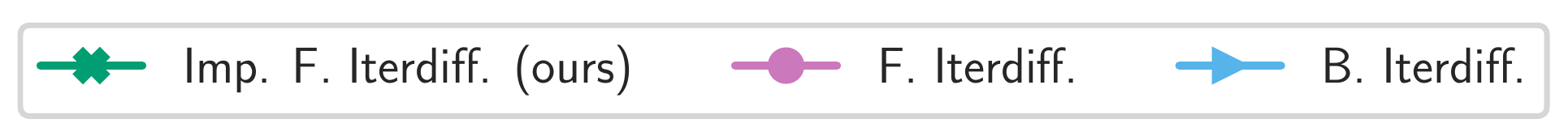}
    \centering
      \includegraphics[width=\figsize\linewidth]{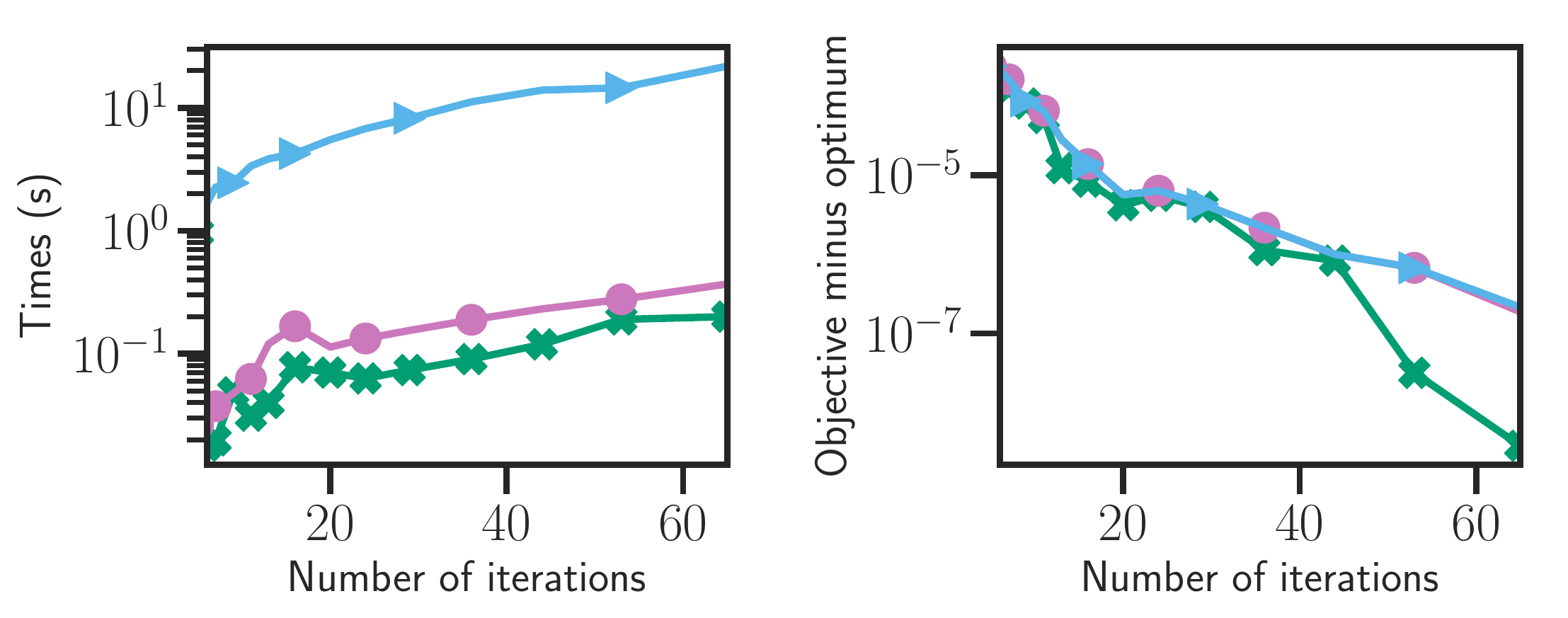}
  \caption{\textbf{Time to compute a single gradient} (Synthetic data, Lasso, $n, p = 1000, 2000$).
  Influence on the number of iterations of BCD (in the inner optimization problem of \Cref{eq:bilevel_opt}) on the computation time (left) and the distance to ``optimum'' of the gradient $ \nabla_\lambda \cL(\lambda) $(right) for the Lasso estimator.
  The ``optimum'' is here the gradient given by implicit differentiation (\Cref{alg:compute_jac_implicit_diff}).
  }
  \label{fig:intro_influ_niter}
\end{figure}
%
%
As an illustration, \Cref{fig:intro_influ_niter} shows the times of computation of a single gradient $\nabla_\lambda \cL(\lambda)$ and the distance to ``optimum'' of this gradient as a function of the number of iterations in the inner optimization problem for the forward iterative differentiation (\Cref{alg:compute_jac_forward_diff_bcd}), the backward iterative differentiation (\Cref{alg:compute_jac_backward_iter_diff}), and the proposed algorithm (\Cref{alg:compute_jac_imp_forward_iter_diff}).
The backward iterative differentiation is several order of magnitude slower than
the forward and our implicit forward method.
Moreover, once the support has been identified (after 20 iterations) the
proposed implicit forward method converges faster than other methods.
Note also that in \Cref{prop:closed_form_jac_lasso,prop:convergence_iterdiff} the Jacobian for the Lasso only depends on the \emph{support} (\ie the indices of the non-zero coefficients) of the regression coefficients $\hbeta^{(\lambda)}$.
In other words, once the support of $\hbeta^{(\lambda)}$ is correctly
identified, even if the value of the non-zeros coefficients are not correctly
estimated, the Jacobian is exact, see \citet{Sun_Jeong_Nutini_Schmidt2019} for
support identification guarantees.

{\fontsize{4}{4}\selectfont
\begin{algorithm}[t]
\SetKwInOut{Input}{input}
\SetKwInOut{Init}{init}
\SetKwInOut{Parameter}{param}
\caption{\textsc{Imp. F. Iterdiff.} (proposed)}
\Input{$
    X \in \bbR^{n \times p},
    y \in \bbR^{n},
    \lambda \in \bbR,
    n_{\text{iter}},
    n_{\text{iter\_jac}} \in \bbN $
}
\Init{
    $\jac = 0$}
\tcp{sequentially compute coef. \& Jacobian }

\If{Lasso}{
    Get $\hbeta = Lasso(X, y, \lambda, n_{\text{iter}})
    $ and its support $\hat S$. \\
    $dr = - X_{:, \hat S} \jac_{\hat S}$
    \tcp*[r]{trick for cheap updates}
    }
\If{wLasso}{
    Get $\hbeta = wLasso(X, y, \lambda, n_{\text{iter}})
    $ and its support $\hat S$. \\
    $dr = - X_{:, \hat S} \jac_{\hat S,\hat S}$
}
    \For{$k = 0,\dots, n_{\text{iter\_jac}} - 1$
    }{
        \For{$j \in \hat S$}{
            \If{Lasso}{
                $\jac_{old} = \jac_{j}$
                \tcp*{trick for cheap update}
                \tcp{diff. \Cref{eq:coordinate_descent_lasso} \wrt $\lambda$ }
                $\jac_{j} \pluseq
                 \frac{ X_{:, j}^\top dr}{\norm{X_{:,j}}^{2}}
                 - \frac{n e^{\lambda}}{\norm{X_{:,j}}^{2}} \sign \hbeta_j $
                 \tcp*[r]{$\bigo(n)$ }

                 $dr \minuseq X_{:, j} (\jac_{j, :} - \jac_{old})$
                 \tcp*[r]{$\bigo(n)$}
            }
            \If{wLasso}{
                $\jac_{old} = \jac_{j, :}$
                \tcp*{trick for cheap update}

                \tcp{diff.  \Cref{eq:coordinate_descent_lasso} \wrt $\lambda$ }
                $\jac_{j, \hat S} \pluseq
                 \frac{1}{\norm{X_{:,j}}^{2}} X_{:, j}^\top dr$
                \tcp*[r]{$\bigo(n \times \hat s)$}

                $\jac_{j, j} \minuseq \frac{n e^{\lambda_j}}{\norm{X_{:,j}}^{2}}  \sign \hbeta_j $
                \tcp*[r]{$\bigo(1)$}

                $dr \minuseq X_{:, j} \otimes(\jac_{j, :} - \jac_{old})$
                \tcp*[r]{$\bigo(n \times \hat s)$}
                }
        }
    }
\Return{
    $\hbeta, \jac$
    }
\label{alg:compute_jac_imp_forward_iter_diff}
\end{algorithm}
}

%
{\centering
\begin{table*}[t]
  \caption{Summary of cost in time and space for each method}
  \label{tab:summary_costs}
  \centering
  \begin{tabular}{lc|cc|cc}
    \toprule
     Mode
     & Computed
     & Space
     & Time
     & Space
     & Time \\
     & quantity
     & (Lasso)
     & (Lasso)
     & (wLasso)
     & (wLasso)\\
    \midrule
    \fiter
    & $\jac$
    & $\bigo(p)$
    & $\bigo(2npn_{\text{iter}})$
    & $\bigo(p^2)$
    &  $\bigo(np^2n_{\text{iter}})$
    \\
    B. Iterdiff.
    & $\jac^\top v$
    & $\bigo(2pn_{\text{iter}} )$
    & $\bigo(npn_{\text{iter}} + np^2n_{\text{iter}})$
    & $\bigo(p^2n_{\text{iter}})$
    & $\bigo(npn_{\text{iter}} + np^2n_{\text{iter}})$
    \\
    Implicit
    & $\jac^\top v$
    &  $\bigo(p)$
    & $\bigo(npn_{\text{iter}} + \hat s^3)$
    & $\bigo(p + \hat s^2)$
    & $\bigo(npn_{\text{iter}} + \hat s^3)$
    \\
    \our
    & $\jac$
    & $\bigo(p)$
    & $\bigo(npn_{\text{iter}} + n \hat sn_{\text{iter\_jac}})$
    & $\bigo(p + \hat s^2)$
    & $\bigo(npn_{\text{iter}} +  n \hat s^2 n_{\text{it\_jac}})$
    \\
    \bottomrule
  \end{tabular}
\end{table*}
}


\section{Experiments}
\label{sec:expes}

Our Python code is released as an open source package:
\url{https://github.com/QB3/sparse-ho}.
All the experiments are written in Python using Numba \citep{Lam_Pitrou_Seibert15} for the critical parts
such as the BCD loop.
We compare our gradient computation technique against other competitors (see the competitors section) on the HO problem (\Cref{eq:bilevel_opt}).

\textbf{Solving the inner optimization problem.}
Note that our proposed method, \ourfull, has the appealing property that it can be used with any solver.
For instance for the Lasso one can combine the proposed algorithm with state of the art solver such as \citet{Massias_Gramfort_Salmon18} which
 would be tedious to combine with iterative differentiation methods.
However for the comparison to be fair, for all methods we have used the same vanilla BCD algorithm (recalled in \cref{alg:bcd_lasso}).
We stop the Lasso-types solver when $\frac{f(\beta^{(k+1)})-f(\beta^{(k)})}{f(0)}< \epsilon^\text{tol}\enspace, $ where $f$ is the cost function of the Lasso or wLasso and $\epsilon^\text{tol}$ a given tolerance.
The tolerance is fixed at $\epsilon^\text{tol}=10^{-5}$ for all methods throughout the different benchmarks.

\textbf{Line search.}
For each hypergradient-based method, the gradient step is combined with a line-search strategy following the work of \citet{Pedregosa16}\footnote{see \url{https://github.com/fabianp/hoag} for details}.

\textbf{Initialization.}
Since the function to optimize $\cL$ is not convex, initialization plays a crucial role in the final solution as well as the convergence of the algorithm.
For instance, initializing $\lambda=\lambda_{\rm init}$ in a flat zone of $\cL(\lambda)$ could lead to slow convergence.
In the numerical experiments, the Lasso is initialized with
$\lambda_{\rm init} = \lambda_{\text{max}} - \log(10)$,
where $\lambda_{\text{max}}$ is the smallest $\lambda$ such that $0$ is a solution of \Cref{pb:lasso}.

\textbf{Competitors.}
In this section we compare the empirical performance of \ourfull
algorithm to different competitors.
Competitors are divided in two categories. Firstly, the ones relying on hyperparameter gradient:

\begin{itemize}
    \item \textbf{\our}: \ourfull (proposed) described in \Cref{alg:compute_jac_imp_forward_iter_diff}.
    \item \textbf{Implicit}:
    \implicitfull, which requires solving a $\hat s \times \hat s$ linear system as described in \Cref{alg:compute_jac_implicit_diff}.
    \item \textbf{F. Iterdiff.}:
    \forwardfull \citep{Deledalle_Vaiter_Fadili_Peyre14,Franceschi_Donini_Frasconi_Pontil17} which jointly computes the regression coefficients $\hat \beta$ as well as the Jacobian $\hat \jac$ as shown in \Cref{alg:compute_jac_forward_diff_bcd}.
\end{itemize}

Secondly, the ones not based on hyperparameter gradient:
\begin{itemize}
  \item \textbf{Grid-search}: as recommended by \citet{Friedman_Hastie_Tibshirani10}, we use $100$ values on a uniformly-spaced grid from $\lambda_{\text{max}}$ to $\lambda_{\text{max}} - 4 \log(10)$.
  \item \textbf{Random-search}: we sample uniformly at random $100$ values taken on the same interval as for the Grid-search $[\lambda_{\text{max}} - 4 \log(10) ; \lambda_{\text{max}}]$, as suggested by \citet{Bergstra13}.
  \item \textbf{Lattice Hyp.}:
  lattice hypercube sampling
  \citep{Bousquet_Gelly_Kurach_Teyaud_Vincent17},
  combines the idea of grid-search and random-search.
  We used the sampling scheme of
  \citet{Bouhlel_Hwang_Bartoli_Lafage_Morlier_Martins2019}
  and their code
  \footnote{https://github.com/SMTorg/smt}
  to sample the points to evaluate the function on.
  \item \textbf{Bayesian}: \sbmofull (SMBO) using a Gaussian process to model the objective function. We used the implementation of \citet{Bergstra13}.\footnote{https://github.com/hyperopt/hyperopt} The constraints space for the hyperparameter search was set in $[\lambda_{\text{max}} - 4 \log(10) ; \lambda_{\text{max}}]$, and the expected improvement (EI) was used as aquisition function.
\end{itemize}
The cost and the quantity computed by each algorithm can be found in \Cref{tab:summary_costs}.
The \backwardfull \citep{Domke12} is not included in the benchmark in \Cref{fig:Lasso_train_test_perf} since it was several orders of magnitude slower than the other techniques (see \Cref{fig:intro_influ_niter}). This is due to the high cost of the BCD algorithm in backward mode, see \Cref{tab:summary_costs}.

\subsection{Application to held-out loss}\label{sub:crossval}
%
When using the held-out loss, each dataset $(X, y)$ is split in 3 equal parts: the training set $(X^{\text{train}}, y^\text{train})$, the validation set $(X^{\text{val}}, y^\text{val})$ and the test set $(X^{\text{test}}, y^\text{test})$.
%

(\textit{Lasso, held-out criterion}).
For the Lasso and the held-out loss, the bilevel optimization \Cref{eq:bilevel_opt} reads:
\begin{align}\label[pb_multiline]{pb:bilevel_opt_cv}
  \argmin_{\lambda \in \bbR}
  &
  \normin{y^{\text{val}} - X^{\text{val}} \hat \beta^{(\lambda)} }^{2}
  \\
  \st \hat \beta^{(\lambda)} &\in
  \argmin_{\beta \in \bbR^p}
    {
      \tfrac{1}{2n} \normin{y^{\text{train}} - X^{\text{train}} \beta}_2^2+ e^\lambda \normin{\beta}_1
    }
   \enspace.\nonumber
\end{align}
\Cref{fig:Lasso_train_test_perf} (top) shows on 3 datasets (see \Cref{app:dataset} for dataset details) the distance to the ``optimum'' of $ \normin{y^{\text{val}} - X^{\text{val}} \hat \beta^{(\lambda)} }^{2}$ as a function of time. Here the goal is to find $\lambda$ solution of \Cref{pb:bilevel_opt_cv}.
The ``optimum'' is chosen as the minimum of $  \normin{y^{\text{val}} - X^{\text{val}} \hat \beta^{(\lambda)} }^{2}$ among all the methods.
\Cref{fig:Lasso_train_test_perf} (bottom) shows the loss $ \normin{y^{\text{test}} - X^{\text{test}} \hat \beta^{(\lambda)} }^{2}$ on the test set (independent from the training set and the validation set). This illustrates how well the estimator generalizes.
Firstly, it can be seen that on all datasets the proposed \ourfull outperforms \forwardfull which illustrates \Cref{prop:convergence_iterdiff} and corroborates the cost of each algorithm in \Cref{tab:summary_costs}.
Secondly, it can be seen that on the \emph{20news} dataset (\Cref{fig:Lasso_train_test_perf}, top) the \implicitfull (\Cref{alg:compute_jac_implicit_diff}) convergence is slower than \ourfull, \forwardfull,
and even slower than the grid-search.
In this case, this is due to the very slow convergence of the conjugate gradient algorithm~\citep{Nocedal_Wright06} when solving the ill-conditioned linear system in \Cref{alg:compute_jac_implicit_diff}.
%
\def \figsize {1}
\def \figprop {2}
\def \colprop {0.5}

\begin{figure*}[tb]
    \centering
    \begin{subfigure}[b]{0.9\textwidth}
        \centering
        \includegraphics[width=\figsize\linewidth]{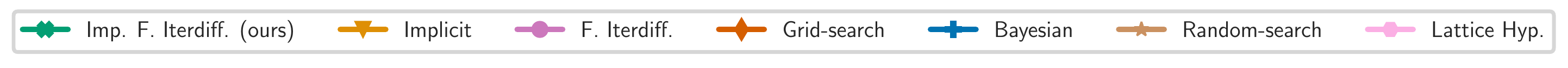}
        \includegraphics[width=\figsize\linewidth]{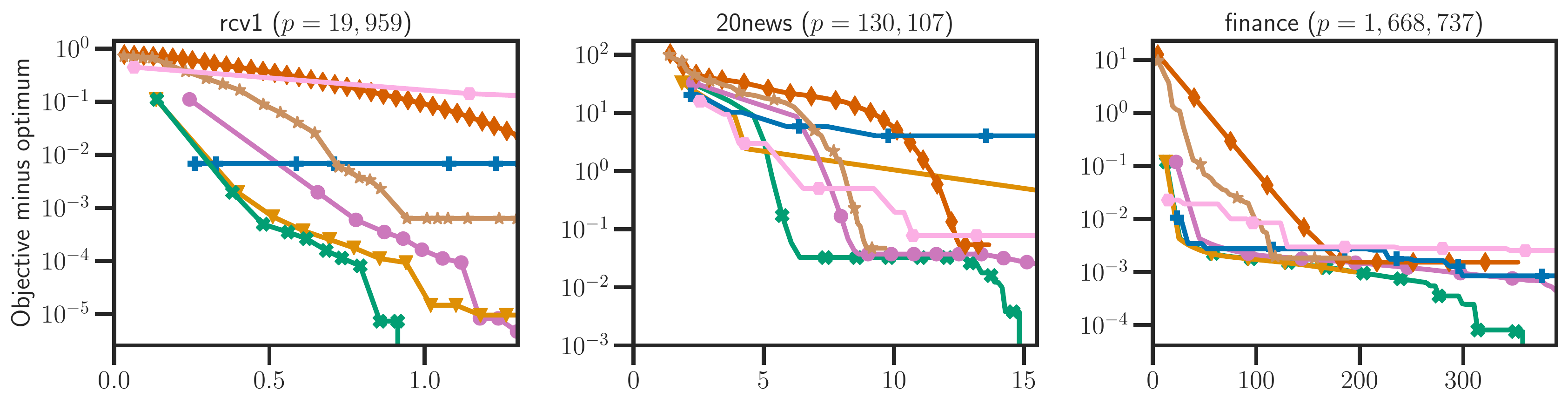}
    \end{subfigure}%

    \begin{subfigure}[b]{0.9\textwidth}
        \centering
        \includegraphics[width=\figsize\linewidth]{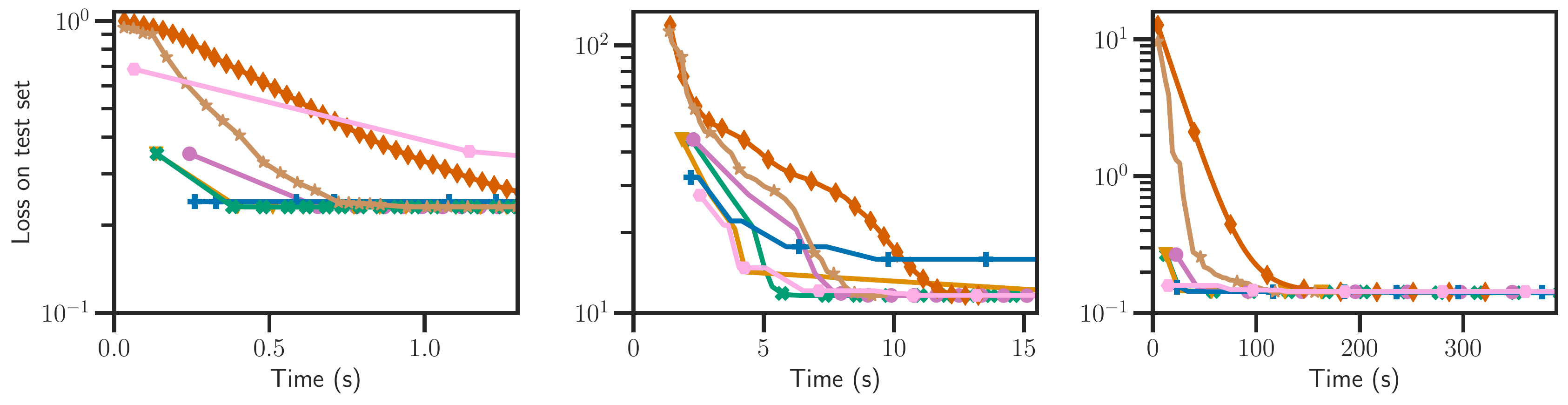}
    \end{subfigure}
    \caption{\textbf{Computation time for the HO of the Lasso on real data.} Distance to ``optimum'' (top) and performance (bottom) on the test set for the Lasso for 3 different datasets: \emph{rcv1}, \emph{20news} and \emph{finance}.}
    \label{fig:Lasso_train_test_perf}
\end{figure*}
(\textit{MCP, held-out criterion}).
We also applied our algorithm on an estimator based on a non-convex penalty: the MCP \citep{Zhang10} with 2 hyperparameters.
Since the penalty is non-convex the estimator may not be continuous \wrt hyperparameters and the theory developed above does not hold.
However experimentally \ourfull outperforms \forwardfull for the HO, see \Cref{app:sec_mcp} for full details.
%


\subsection{Application to another criterion: SURE}\label{sub:sugar}
%
Evaluating models on held-out data makes sense if the design is formed
from random samples as it is often considered in supervised learning.
However, this assumption does not hold for certain kinds of applications in signal or image processing. For these applications, the held-out loss cannot be used as the criterion for optimizing the hyperparameters of a given model.
In this case, one may use a proxy of the prediction risk, like the Stein Unbiased Risk Estimation (SURE, \citet{Stein81}).
The SURE is an unbiased estimator of the prediction risk under weak differentiable
conditions. The drawback of this criterion is that it requires the knowledge of the variance of the noise.
The SURE is defined as follows:
    $\text{SURE}(\lambda)
    = \normin{y - X\hat{\beta}^{(\lambda)} }^{2} \!\!- n \sigma^{2} + 2\sigma^{2}{\text{dof}(\hat \beta^{(\lambda)})} \enspace,$
where the degrees of freedom (dof \citealt{Efron86}) is defined as
$
  {\text{dof}(\hat \beta^{(\lambda)})} =
  \sum_{i=1}^n \cov(y_i, (X \hat \beta^{(\lambda)})_i) / \sigma^2 \enspace.$
The dof can be seen a measure of the complexity of the model, for instance for the Lasso $\text{dof}(\hat \beta^{(\lambda)}) = \hat s$, see \citet{Zou_Hastie_Tibshirani07}.
The SURE can thus be seen as a criterion trading data-fidelity against model complexity.
However, the dof is not differentiable (not even continuous in the Lasso case), yet it is possible to construct a weakly differentiable approximation of it based on Finite Differences Monte-Carlo (see \citealt{Deledalle_Vaiter_Fadili_Peyre14} for full details), with $\epsilon > 0$ and $\delta \sim \mathcal{N}(0, \Id_{n})$:
\begin{equation}
    {\text{dof}_{\text{FDMC}}}(y, \lambda, \delta, \epsilon)=\tfrac{1}{\epsilon}\langle X \hat \beta^{(\lambda)} (y+ \epsilon\delta)- X\hat \beta^{(\lambda)}(y), \delta \rangle \enspace. \nonumber
\end{equation}
We use this smooth approximation in the bi-level optimization problem to find the best hyperparameter.
The bi-level optimization problem then reads:
\begin{align}\label[pb_multiline]{pb:bilevel_opt_sure}
  &\argmin_{\lambda \in \bbR}
  \normin{y - X \hat \beta^{(\lambda)} }^{2}
  + 2\sigma^{2}{\text{dof}}_{\text{FDMC}}(y, \lambda, \delta, \epsilon)
  \\
  &\st \hat \beta^{(\lambda)}(y) \in
  \argmin_{\beta \in \bbR^p}
    {
      \tfrac{1}{2n} \normin{y- X \beta}_2^2+ e^\lambda \normin{\beta}_1
    } \nonumber
  \\
  &\phantom{\st} \hat \beta^{(\lambda)}(y + \epsilon \delta) \in
  \argmin_{\beta \in \bbR^p}
    {
      \tfrac{1}{2n} \normin{y + \epsilon \delta - X \beta}_2^2+ e^\lambda \normin{\beta}_1
    } \nonumber
\end{align}
Note that solving this problem requires the computation of two (instead of one for the held-out loss) Jacobians \wrt $\lambda$ of the solution $\hat \beta^{(\lambda)}$ at the points $y$ and $y + \epsilon \delta$.

(\textit{Lasso, SURE criterion}).
To investigate the estimation performance of the \ourfull in comparison to the competitors described above, we used as metric the (normalized) Mean Squared Error (MSE) defined as $\text{MSE} \eqdef {\normin{\hat{\beta} - \beta^{*}}^{2}}/{\norm{\beta^{*}}^{2}}$.
The entries of the design matrix $X \in \bbR^{n \times p}$ are \iid random Gaussian variables $\mathcal{N}(0,1)$.
The number of rows is fixed to $n=100$.
Then, we generated $\beta^{*}$ with 5 non-zero coefficients equals to $1$.
The vector $y$ was computed by adding to $X\beta^{*}$ additive Gaussian noise controlled by the Signal-to-Noise Ratio:
$  \text{SNR} \eqdef {\norm{X\beta^*}}/{\norm{y - X\beta^*}}$ (here $\text{SNR}=3$).
Following \citet{Deledalle_Vaiter_Fadili_Peyre14}, we set $\epsilon = 2\sigma / n^{0.3}$.
We varied the number of features $p$ between 200 and 10,000 on a linear grid of size 10.
For a fixed number of features, we performed 50 repetitions and each point of the curves represents the mean of these repetitions.
Comparing efficiency in time between methods is difficult since they are not directly comparable.
Indeed, grid-search and random-search discretize the HO space whereas others methods work in the continuous space which is already an advantage.
However, to be able to compare the hypergradient methods and possibly compare them to the others, we computed the total amount of time for a method to return its optimal value of $\lambda$.
In order to have a \emph{fair} comparison, we compared $50$ evaluations of the line-search for each hypergradient methods, 50 evaluations of the Bayesian methods and finally 50 evaluations on fixed or random grid.
We are aware that the cost of each of these evaluations is not the same but it allows to see that our method stays competitive in time with optimizing one parameter.
Moreover we will also see that our method scales better with a large number of hyperparameters to optimize.

\Cref{fig:lasso_estimation} shows the influence of the number of features on the relative MSE (ie. MSE of a method minus the MSE of our implicit forward method) and the computation time.
First, MSE of all gradient based methods is lower than the other methods which means that $\hat\beta^{(\lambda)}$ leads to a better estimation when $\lambda$ is chosen via the gradient based methods.
This illustrates that continuous optimization for hyperparameter selection leads to better estimation performance than discrete or Bayesian optimization.
Yet, the running time of our proposed method is the lowest of all hypergradient-based strategies and competes with the grid-search and the random-search.
\def \figsize {1}
\def \figprop {2}
\def \colprop {0.5}
\begin{figure}[t]
        \includegraphics[width=\figsize\linewidth]{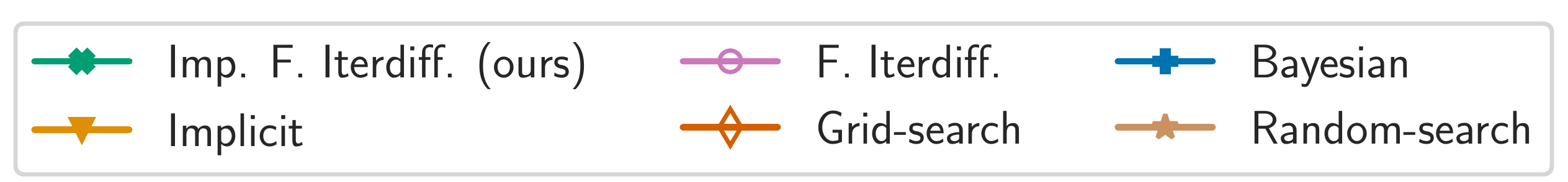}
        \centering
         \includegraphics[width=\figsize\linewidth]{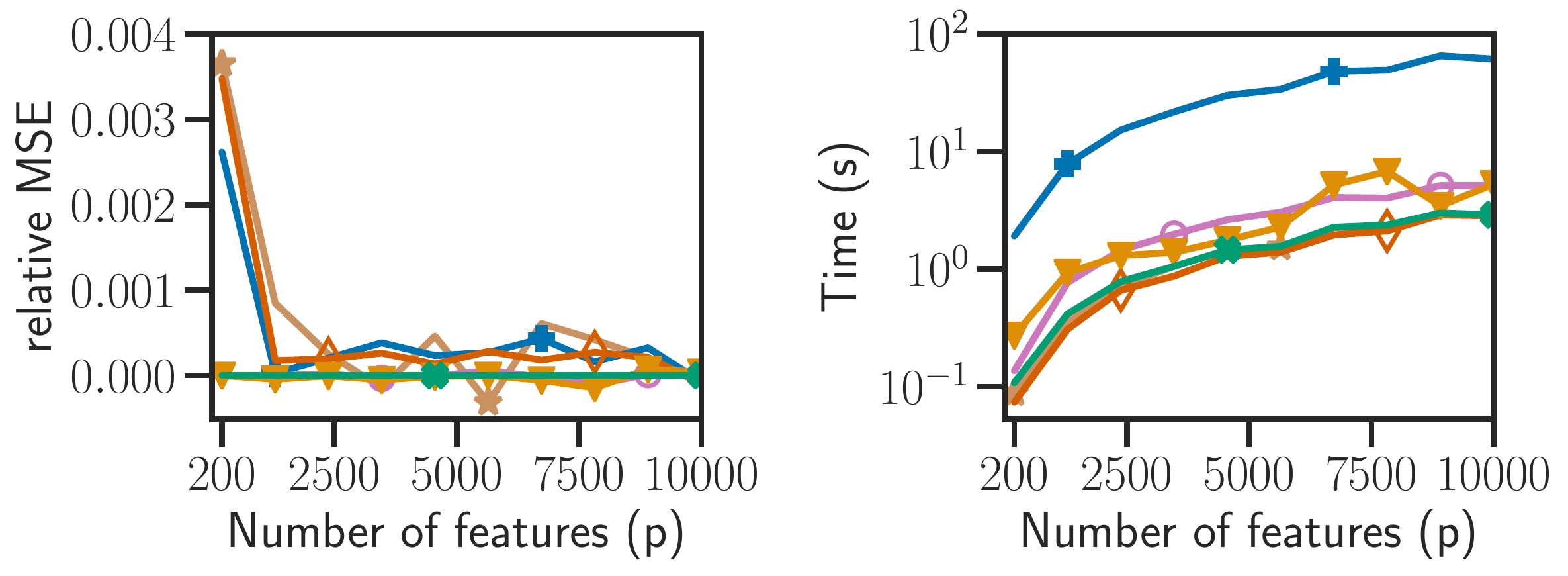}
\caption{\textbf{Lasso: estimation performance.}
Estimation relative Mean Squared Error (left) and running time (right) as a function of the number of features for the Lasso model.}
\label{fig:lasso_estimation}
\end{figure}

(\textit{Weighted Lasso vs Lasso, SURE criterion}).
As our method leverages the sparsity of the solution, it can be used for HO with a large number of hyperparameters, contrary to classical \forwardfull.
The weighted Lasso (wLasso, \citealt{Zou06}) has $p$ hyperparameters and was introduced to reduce the bias of the Lasso.
However setting the $p$ hyperparameters is impossible with grid-search.

\Cref{fig:Lasso_vs_wLasso} shows the estimation MSE and the running time of the different methods to obtain the hyperparameter values as a function of the number of features used to simulate the data.
The simulation setting is here the same as for the Lasso problems investigated in \Cref{fig:lasso_estimation} ($n=100$, $\text{SNR}=3$).
We compared the classical Lasso estimator and the weighted Lasso estimator where the regularization hyperparameter was chosen using \ourfull and the forward iterative differentiation as described in \Cref{alg:compute_jac_forward_diff_bcd}.
\Cref{eq:bilevel_opt} is not convex for the weighted Lasso and a descent algorithm like ours can be trapped in local minima, crucially depending on the starting point $\lambda_{\rm init}$.
To alleviate this problem, we introduced a regularized version of \Cref{eq:bilevel_opt}:
\begin{align}
    \label{eq:l2_bilevel_opt}
    \argmin_{\lambda \in \bbR}
    & \quad
    \mathcal{C} \left (\hbeta^{(\lambda)}  \right)
    +
    \gamma \sum_j^p \lambda_j^2
    \nonumber \\
    &\st \hbeta^{(\lambda)} \in \argmin_{\beta \in \bbR^p}
      {\eqdef\psi(\beta,\lambda)}
     \enspace.
\end{align}
The solution obtained by solving \Cref{eq:l2_bilevel_opt} is then used as the initialization $\lambda^{(0)}$ for our algorithm. In this experiment the regularization term is constant $\gamma = C(\beta^{(\lambda_{\text{max}})})/10$. We see in \Cref{fig:Lasso_vs_wLasso} that the weighted Lasso gives a lower MSE than the Lasso and allows for a better recovery of $\beta^{*}$.
This experiment shows that the amount of time needed to obtain the vector of hyperparameters of the weighted Lasso via our algorithm is in the same range as for obtaining the unique hyperparameter of the Lasso problem.
It also shows that our proposed method is much faster than the \textit{naive} way of computing the Jacobian using forward or backward iterative differentiation. The \implicitfull method stays competitive for the wLasso due to the small support of the solution and hence a small matrix to inverse. 
A maximum running time threshold was used for this experiment checking the running time at each line-search iteration, explaining why the \forwardfull and \backwardfull of the wLasso does not explode in time on \Cref{fig:Lasso_vs_wLasso}.

\def \figsize {1}
\def \figprop {2}
\def \colprop {0.5}

\begin{figure}[t]
        \includegraphics[width=\figsize\linewidth]{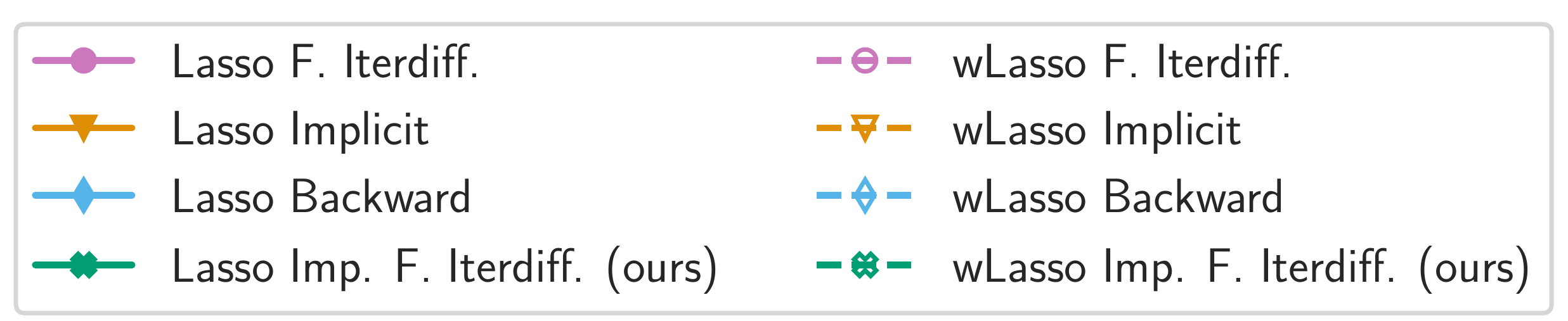}
        \centering
         \includegraphics[width=\figsize\linewidth]{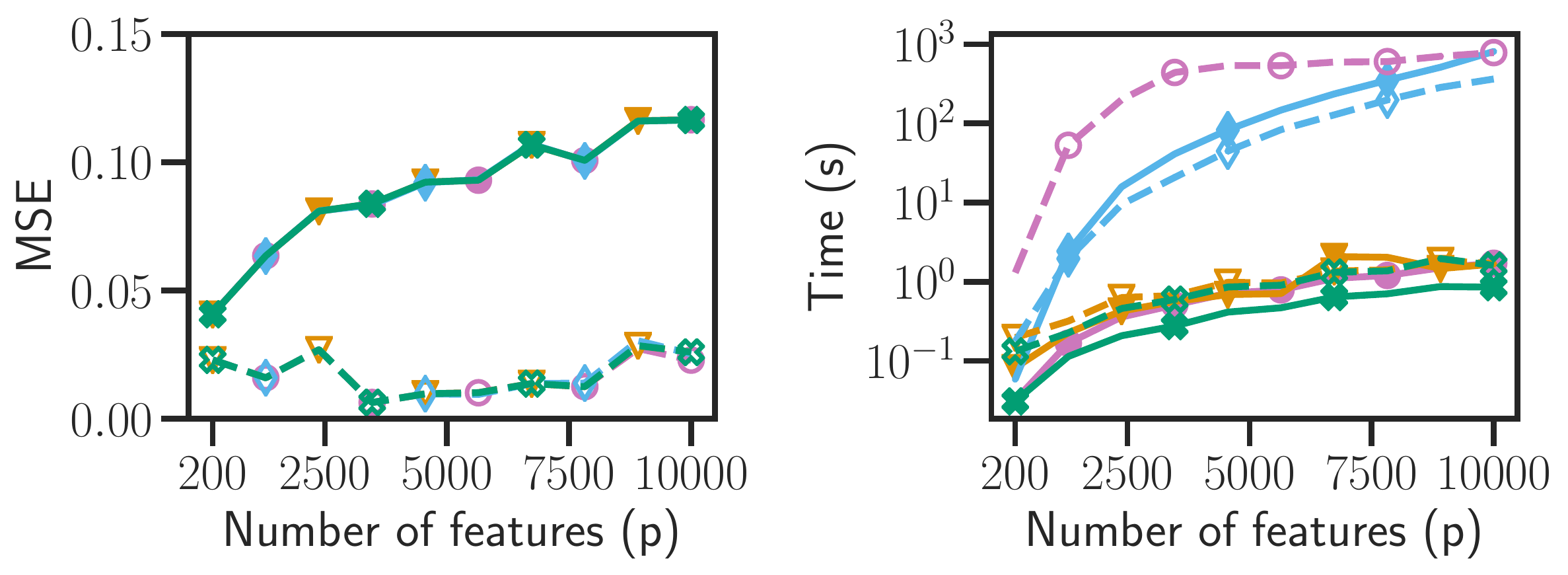}
\caption{\textbf{Lasso vs wLasso.}
\label{fig:Lasso_vs_wLasso}
Estimation Mean Squared Error (left) and running (right) of competitors as a function of the number of features for the weighted Lasso and Lasso models.}
\end{figure}


\section*{Conclusion}
In this work we studied the performance of several methods to select hyperparameters of Lasso-type estimators showing results for the Lasso and the weighted Lasso, which have respectively one or $p$ hyperparameters.
We exploited the sparsity of the solutions and the specific structure of the iterates of forward differentiation, leading to our \ourfull algorithm that computes efficiently the full Jacobian of these estimators \wrt the hyperparameters. This allowed us to select them through a standard gradient descent and have an approach that scales to a high number of hyperparameters. Importantly, contrary to a classical implicit differentiation approach, the proposed algorithm does not require solving a linear system.
Finally, thanks to its two steps nature, it is possible to leverage in the first step the availability of state-of-the-art Lasso solvers that make use of techniques such as active sets or screening rules. Such algorithms, that involve calls to inner solvers run on subsets of features, are discontinuous \wrt hyperparameters which would significantly challenge a single step approach based on automatic differentiation.

\paragraph{Acknowledgments}
This work was funded by ERC Starting Grant SLAB ERC-StG-676943 and ANR GraVa ANR-18-CE40-0005.

\clearpage
\bibliographystyle{plainnat}
\bibliography{references_all}
\clearpage
\appendix
\onecolumn

\section{Proofs}
\label{sec:proofs}

\subsection{Proof of \Cref{prop:closed_form_jac_lasso}}
\label{sub:proof_of_prop:closed_form_jac_lasso}

We start by a lemma on the weak derivative of the soft-thresholding.
\begin{lemma}\label{app:lemma_st_deriv}
    The soft-thresholding $\ST: \bbR \times \bbR^{+} \mapsto \bbR$ defined by $\ST(t,\tau) = \sign(t) \cdot(|t|-\tau)_{+}$ is weakly differentiable with weak derivatives
    \begin{align}\label{eq:st_weak_t}
        \partial_1 \ST
        (t, \tau) =
        \ind_{\{\abs{t} > \tau\}} \enspace ,
    \end{align}
    and
    \begin{align}\label{eq:st_weak_tau}
        \partial_2 \ST(t, \tau)
        =
        - \sign(t) \cdot \ind_{\{\abs{t} > \tau\}} \enspace,
    \end{align}
   where
   \begin{align}
   \ind_{\{\abs{t} > \tau\}} =
   \begin{cases}
       1, & \text{ if} \abs{t} > \tau, \\
       0, & \text{ otherwise}.
       \end{cases}
   \end{align}
\end{lemma}
\begin{proof}
    See~\citep[Proposition 1]{Deledalle_Vaiter_Fadili_Peyre14}
\end{proof}
\begin{proof} (\Cref{prop:closed_form_jac_lasso}, Lasso ISTA)
    The soft-thresholding is differentiable almost everywhere (a.e.), thus \Cref{eq:fixed_point_lasso} can be differentiated a.e. thanks to the previous lemma, and for any $\alpha > 0$
    \begin{align}
        \hat \jac_ =
            \begin{pmatrix}
                \ind_{\left\{\abs{\hbeta_1} > 0\right\}} \\
                \vdots\\
                \ind_{\left\{\abs{\hbeta_p} > 0\right\}}
            \end{pmatrix}
            \nonumber \odot \left (
                \Id_p - \frac{1}{\alpha} X^\top X \right ) \hat \jac_
        - \frac{n e^\lambda}{\alpha}
        \begin{pmatrix}
            \sign(\hbeta_1)\ind_{\left\{\abs{\hbeta_1} > 0 \right\}} \\
            \vdots\\
            \sign(\hbeta_p)\ind_{\left\{\abs{\hbeta_p} > 0\right\}}
        \end{pmatrix}
        \enspace .
    \end{align}
    Inspecting coordinates inside and outside the support of $\hbeta$ leads to:
    \begin{align}
         \left\{\begin{array}{lll}
        \hat \jac_{\hat S^c} & = & 0\\
        \hat \jac_{\hat S}   & = &
        \hat \jac_{\hat S} - \frac{1}{\alpha} X_{:, \hat S}^\top X_{:, \hat S} \hat \jac_{\hat S}
        - \frac{n e^\lambda}{\alpha} \sign{\hat \beta_{\hat S}}
        \enspace . \label{eq:diff_fix_point}
        \end{array}
        \right.
    \end{align}
    Rearranging the term of \Cref{eq:diff_fix_point} it yields:
    \begin{align}
        X_{:, \hat S}^\top X_{:, \hat S} \hat \jac_{\hat S}
        &= - n e^\lambda \sign{\hat \beta_{\hat S}} \\
        \hat \jac_{\hat S}
        &= - n e^\lambda
        \left ( X_{:, \hat S}^\top X_{:, \hat S} \right )^{-1} \sign{\hat \beta_{\hat S}} \enspace.
    \end{align}

    (\Cref{prop:closed_form_jac_lasso}, Lasso BCD)

    The fixed point equations for the BCD case is
    \begin{equation}
        \label{eq:BCD_fixed}
     \hat\beta_{j} = \ST \left (\hat\beta_{j} - \frac{1}{\norm{X_{:j}}^{2}_{2}}X^\top_{:j}(X\hat\beta_{j}-y), \frac{ne^{\lambda}}{\norm{X_{:j}}_{2}^{2}} \right )\enspace.
    \end{equation}

    As before we can differentiate this fixed point equation \Cref{eq:BCD_fixed}
    \begin{equation}
        \hat \jac_{j} =
        \ind_{\left\{\abs{\hbeta_j} > \tau \right \}}
        \cdot
        \left(\hat\jac_{j} - \frac{1}{\norm{X_{:j}}^{2}_{2}}X^{\top}_{:j}X\hat\jac\right)
        -
        \frac{ne^{\lambda}}{\norm{X_{:j}}_{2}^{2}} \sign{(\hat\beta_{j})}\ind_{\left\{\abs{\hbeta_j} > \tau \right \}}\enspace,
    \end{equation}
    leading to the same result.
\end{proof}

%
\subsection{Proof of \Cref{prop:convergence_iterdiff} in the ISTA case}
\label{app:sub_conv_jac_ista}
%
\begin{proof} (Lasso case, ISTA)
    In \Cref{alg:compute_jac_forward_diff_bcd}, $\beta^{(k)}$ follows ISTA steps, thus $(\beta^{(k)})_{l \in \bbN}$ converges toward the solution of the Lasso $\hbeta$.
    Let $\hat S$ be the support of the Lasso estimator $\hbeta$, and $\nu^{(\hat S)} > 0$ the smallest eigenvalue of $X_{:, \hat S}^\top X_{:, \hat S}$.
    Under uniqueness assumption proximal gradient descent (\aka ISTA) achieves sign identification \citep{Hale_Yin_Zhang08}, \ie there exists $k_0 \in \bbN$ such that for all $k \geq k_0 - 1$:

    \begin{align}
        \sign \beta^{(k+1)}
        &= \sign  \hbeta
        \enspace .
    \end{align}
    Recalling the update of the Jacobian $\jac$ for the Lasso solved with ISTA is the following:
    \begin{align}
        \jac^{(k+1)} = &
        \left | \sign  \beta^{(k+1)} \right |
        \odot \left ( \Id - \frac{1}{\norm{X}_2^2} X^\top X \right ) \jac^{(k)} \nonumber - \frac{n e^\lambda}{\norm{X}_2^2} \sign  \beta^{(k+1)}
        \enspace,
    \end{align}
    it is clear that $\jac^{(k)}$ is sparse with the sparsity pattern $\beta^{(k)}$ for all $k \geq k_0$.
    Thus we have that for all $k \geq k_0$:
    \begin{align}
        \jac_{\hat S}^{(k+1)}
        &= \jac_{\hat S}^{(k)} - \frac{1}{\norm{X}_2^2} X_{:, \hat S}^\top X \jac^{(k)} - \frac{n e^\lambda}{\norm{X}_2^2} \sign \hbeta_{\hat S} \nonumber \\
        &= \jac_{\hat S}^{(k)} - \frac{1}{\norm{X}_2^2} X_{:, \hat S}^\top X_{:, \hat S} \jac_{\hat S}^{(k)} - \frac{n e^\lambda}{\norm{X}_2^2} \sign \hbeta_{\hat S} \nonumber \\
        &= \left ( \Id_{\hat S} - \frac{1}{\norm{X}_2^2} X_{:, \hat S}^\top X_{:, \hat S} \right ) \jac_{\hat S}^{(k)} - \frac{n e^\lambda}{\norm{X}_2^2} \sign \hbeta_{\hat S} . \label{eq:rec_jac_lasso}
    \end{align}
    One can remark that $\hat \jac$ defined in \Cref{eq:closed_form_lasso}, satisfies the following:
    \begin{align}
        \hat \jac_{\hat S}
        = \left ( \Id_{\hat S}
                - \frac{1}{\norm{X}_2^2} X_{:, \hat S}^\top X_{:, \hat S}
            \right ) \hat \jac_{\hat S}
        - \frac{n e^\lambda}{\norm{X}_2^2} \sign \hbeta_{\hat S} \label{eq:jac_hat_lasso} \enspace .
    \end{align}
    Combining \Cref{eq:rec_jac_lasso,eq:jac_hat_lasso} and denoting $\nu^{({\hat S})} > 0$ the smallest eigenvalue of $X_{\hat S}^\top X_{\hat S}$, we have for all $k \geq k_0 $:
    \begin{align}
        \jac_{\hat S}^{(k+1)} - \hat \jac_{\hat S}
        &= \left (
                \Id_{\hat S} - \frac{1}{\norm{X}_2^2} X_{:, \hat S}^\top X_{:, \hat S}
            \right ) \left ( \jac_{\hat S}^{(k)} - \hat \jac_{\hat S} \right ) \nonumber \\
        \normin{\jac_{\hat S}^{(k+1)} - \hat \jac_{\hat S}}_2
        & \leq \left (1 - \frac{\nu^{({\hat S})}}{\norm{X}_2^2} \right ) \normin{\jac_{\hat S}^{(k)} - \hat \jac_{\hat S}}_2 \nonumber \\
        \normin{\jac_{\hat S}^{(k)} - \hat \jac_{\hat S}}_2
        & \leq \left (1 - \frac{\nu^{({\hat S})}}{\norm{X}_2^2} \right )^{k - k_0} \normin{\jac_{\hat S}^{(k_0)} - \hat \jac_{\hat S}}_2 \nonumber \enspace .
    \end{align}
    Thus the sequence of Jacobian $\left (\jac^{(k)} \right )_{k \in \bbN}$ converges linearly to $\hat \jac$ once the support is identified.
\end{proof}

\begin{proof}
    (wLasso case, ISTA) Recalling the update of the Jacobian $\jac\in \bbR^{p\times p}$ for the wLasso solved with ISTA is the following:
    \begin{align}
        \jac^{(k+1)} = &
        \left | \sign  \beta^{(k+1)} \right |
        \odot \left ( \Id - \frac{1}{\norm{X}_2^2} X^\top X \right ) \jac^{(k)} \nonumber \\
        & - \frac{n e^\lambda}{\norm{X}_2^2} \diag \left ( \sign\beta^{(k+1)} \right ) \enspace,
    \end{align}

    The proof follows exactly the same steps as the ISTA Lasso case to show convergence in spectral norm of the sequence $(\jac^{(k)})_{k\in\bbN}$ toward $\hat\jac$.

\end{proof}

\subsection{Proof of \Cref{prop:convergence_iterdiff} in the BCD case}
\label{app:sub_conv_jac_bcd}
%
%
The goal of the proof is to show that iterations of the Jacobian sequence $(\jac^{(k)})_{k \in \bbN}$ generated by the Block Coordinate Descent algorithm (\Cref{alg:compute_jac_forward_diff_bcd}) converges toward the true Jacobian $\hat \jac$.
The main difficulty of the proof is to show that the Jacobian sequence follows a Vector AutoRegressive (VAR, see \citet[Thm. 10]{Massias_Vaiter_Gramfort_Salmon19} for more detail), \ie the main difficulty is to show that there exists $k_0$ such that for all $k \geq k_0$:
\begin{equation}
    \jac^{(k+1)} = A \jac^{(k)} + B \enspace,
\end{equation}
with $A\in \bbR^{p \times p}$ a contracting operator and $B \in \bbR^p$.
We follow exactly the proof of \citet[Thm. 10]{Massias_Vaiter_Gramfort_Salmon19}.

\begin{proof} (Lasso, BCD, forward differentiation (\Cref{alg:compute_jac_forward_diff_bcd}))

Let $j_1, \dots, j_S$ be the indices of the support of $\hat \beta$, in increasing order.
As the sign is identified, coefficients outside the support are 0 and remain 0.
We decompose
the $k$-th epoch of coordinate descent into individual coordinate updates:
Let $\tilde \beta^{(0)} \in \bbR^p$ denote the initialization (i.e., the beginning of the epoch,
    ),
$\tilde \beta^{(1)} = \beta^{(k)}$ the iterate after coordinate $j_1$ has been updated, etc., up to $\tilde \beta^{(S)} $ after coordinate
$j_S$ has been updated, i.e., at the end of the epoch ($\tilde \beta^{(S)} = \beta^{(k+1)}$).
Let $s \in S$, then  $\tilde \beta^{(s)}$ and $\tilde \beta^{(s-1)}$ are equal everywhere, except at coordinate $j_s$:
\begin{align}
    \tilde \jac_{j_s}^{(s)}
    &=
    \tilde \jac_{j_s}^{(s-1)}
    - \frac{1}{\norm{X_{:, j_s}}^2} X_{:, j_s}^\top X \tilde \jac^{(s-1)}
    - \frac{1}{\norm{X_{j_s}}^2} \sign \beta_{j_s}
    \quad \text{after sign identification we have:}
    \nonumber
    \\
    &=
    \tilde \jac_{j_s}^{(s-1)}
    - \frac{1}{\norm{X_{:, j_s}}^2} X_{:, j_s}^\top X_{:, \hat S} \tilde \jac_{\hat S}^{(s-1)}
    - \frac{1}{\norm{X_{:, j_s}}^2} \sign \hat \beta_{j_s}
    \nonumber
    \\
    \tilde \jac_{\hat S}^{(s)}
    &=
    \underbrace{\left (
        \Id_{\hat s}
        - \frac{1}{\norm{X_{:, j_s}}^2} e_{j_s} e_{j_s}^\top
        X_{:, \hat S}^\top X_{:, \hat S}
    \right )}_{A_s}
    \tilde \jac_{\hat S}^{(s-1)}
    - \frac{1}{\norm{X_{:, j_s}}^2} \sign \hat \beta_{j_s}
    \nonumber
    \\
    \left ( X_{:, \hat S}^\top X_{:, \hat S} \right )^{1/2}
    \tilde \jac_{\hat S}^{(s)}
    &=
    \underbrace{\left (
        \Id_{\hat s}
        - \frac{\left ( X_{:, \hat S}^\top X_{:, \hat S} \right )^{1/2}}{\norm{X_{:, j_s}}}
         e_{j_s} e_{j_s}^\top
         \frac{\left ( X_{:, \hat S}^\top X_{:, \hat S} \right )^{1/2}}{\norm{X_{:, j_s}}}
    \right )}_{A^{(s)}}
    \left ( X_{:, \hat S}^\top X_{:, \hat S} \right )^{1/2} \tilde \jac_{\hat S}^{(s-1)}
    - \underbrace{\frac{\left ( X_{:, \hat S}^\top X_{:, \hat S} \right )^{1/2}}{\norm{X_{:, j_s}}^2} \sign \hat \beta_{j_s}}_{b^{(s)}}
    \nonumber
\end{align}
We thus have:
\begin{equation}
    \left ( X_{:, \hat S}^\top X_{:, \hat S} \right )^{1/2} \tilde \jac_{\hat S}^{(\hat s)}
    =
    \underbrace{A^{(\hat s)} \dots A^{(1)} }_{A \in \bbR^{ \hat s \times \hat s}}
    \left ( X_{:, \hat S}^\top X_{:, \hat S} \right )^{1/2}  \jac_{\hat S}^{(1)}
    + \underbrace{A_S \dots A_2 b_1
    + \dots
    + A_S b_{S-1}
    + b_S}_{b \in \bbR^{\hat s}}
    \enspace .
    \nonumber
\end{equation}
After sign identification and a full update of coordinate descent we thus have:
\begin{equation}\label{eq:var_lasso}
    \left ( X_{:, \hat S}^\top X_{:, \hat S} \right )^{1/2}
    \jac_{\hat S}^{(t+1)}
    =
    A  \left ( X_{:, \hat S}^\top X_{:, \hat S} \right )^{1/2}  \jac_{\hat S}^{(t)}
    + b
    \enspace .
\end{equation}
    \begin{lemma}
        \begin{equation}
            \norm{A_s}_2 \leq 1
            \enspace ,
            \nonumber
        \end{equation}
        Moreover if $\norm{A^{(s)} x} = \norm{x}$ then
        \begin{equation}
            x \in
            \text{vect}\left (
                 \frac{\left ( X_{:, \hat S}^\top X_{:, \hat S} \right )^{1/2}}{\norm{X_{:, j_s}}} e_{j_s}
            \right )^\top
        \end{equation}
    \end{lemma}
\begin{proof}
        \begin{equation}
            \frac{\left ( X_{:, \hat S}^\top X_{:, \hat S} \right )^{1/2}}{\norm{X_{:, j_s}}}
         e_{j_s} e_{j_s}^\top
         \frac{\left ( X_{:, \hat S}^\top X_{:, \hat S} \right )^{1/2}}{\norm{X_{:, j_s}}}
         \nonumber
        \end{equation}
        is a symmetric rank 1 matrix, its non-zero eigenvalue is
        \begin{equation}
            e_{j_s}^\top
            \frac{\left ( X_{:, \hat S}^\top X_{:, \hat S} \right )^{1/2}}{\norm{X_{:, j_s}}}
            \frac{\left ( X_{:, \hat S}^\top X_{:, \hat S} \right )^{1/2}}{\norm{X_{:, j_s}}}
            e_{j_s}
            =
            e_{j_s}^\top
            \frac{ X_{:, \hat S}^\top X_{:, \hat S} }{\norm{X_{:, j_s}}^2 }
            e_{j_s} = 1 \enspace .
            \nonumber
        \end{equation}
        An eigenvector associated to this non-zeros eigenvalue is
        \begin{equation}
            \frac{\left ( X_{:, \hat S}^\top X_{:, \hat S} \right )^{1/2}}{\norm{X_{:, j_s}}} e_{j_s}
            \enspace .
            \nonumber
        \end{equation}
        $A_s$ is symmetric and real, is diagonalisable in an orthogonal basis, it has eigenvalue $1$ with multiplicity $\hat s - 1$ and eigenvalue $0$ with multiplicity 1.
        Moreover if $\norm{Ax} = \norm{x}$, then $x \in \text{vect}\left (
             \frac{\left ( X_{:, \hat S}^\top X_{:, \hat S} \right )^{1/2}}{\norm{X_{:, j_s}}} e_{j_s}
        \right )^\top$.
\end{proof}
    \begin{lemma} \label{lem:As}
        \begin{equation}
            \norm{A}_2 < 1
            \enspace .
            \nonumber
        \end{equation}
    \end{lemma}
    \begin{proof}
    $A = A^{(\hat s)} \dots A^{(1)}$
    We have
    \begin{equation}
        \norm{A} \leq
        \underbrace{\normin{A^{(\hat s)}}}_{\leq 1}
        \dots
        \underbrace{\normin{A^{(1)}}}_{\leq 1}
        \leq 1
        \enspace .
        \nonumber
        \end{equation}
    Let $x \in \bbR^{\hat s}$ such that $\norm{Ax} = \norm{x}$, we thus have
    for all $s \in 1, \dots, \hat s$, $\norm{A^{(s)} x} = \norm{x}$.
    Using \Cref{lem:As} we have that
    for all $ s \in 1, \dots, \hat s$
    $x \in \text{vect}\left (
        \frac{\left ( X_{:, \hat S}^\top X_{:, \hat S} \right )^{1/2}}{\norm{X_{:, j_s}}} e_{j_s}
   \right )^\top$,
   \ie
    $x \in \text{vect}\left (
       \left ( X_{:, \hat S}^\top X_{:, \hat S} \right )^{1/2}
   \right )^\top = \{ 0 \}$
   because
   $ X_{:, \hat S}^\top X_{:, \hat S}  \succ 0$
    \end{proof}
   Using \Cref{eq:var_lasso} we have:
    \begin{equation}
        \normin{\jac_{\hat S}^{(t+1)} - \hat \jac}_{(X_{:, \hat S}^{\top} X_{:, \hat S})^{-1}}
        \leq
        \normin{A}_2
        \normin{\jac_{\hat S}^{(t)} - \hat \jac}_{(X_{:, \hat S}^{\top} X_{:, \hat S})^{-1}}
        \enspace ,
    \end{equation}
    with $\norm{A}_2 < 1$, which leads to the desire result.
    Since the recursion of the Jacobian sequences of \Cref{alg:compute_jac_imp_forward_iter_diff} and
    \Cref{alg:compute_jac_imp_forward_iter_diff}
    are the same once the support is identified, the proof of convergence of \Cref{alg:compute_jac_imp_forward_iter_diff} is the same (provided that support identification has been achieved).
\end{proof}

\begin{proof} (wLasso case, BCD)
    As for the Lasso case:
\begin{align}
    \tilde \jac_{j_s,:}^{(s)}
    &=
    \tilde \jac_{j_s,:}^{(s-1)}
    - \frac{1}{\norm{X_{:, j_s}}^2} X_{:, j_s}^\top X \tilde \jac^{(s-1)}
    - \frac{1}{\norm{X_{j_s}}^2} \sign \beta_{j_s}e_{j_{s}} e_{j_{s}}^\top
    \quad \text{after sign identification we have:}
    \nonumber
    \\
    \tilde \jac_{j_s,\hat S}^{(s)} &=
    \tilde \jac_{j_s, \hat S}^{(s-1)}
    - \frac{1}{\norm{X_{:, j_s}}^2} X_{:, j_s}^\top X_{:, \hat S} \tilde \jac_{\hat S, \hat S}^{(s-1)}
    - \frac{1}{\norm{X_{:, j_s}}^2} \sign \hat \beta_{j_s} e_{j_{s}} e_{j_{s}}^\top
    \nonumber
    \\
    (X_{:, \hat S}^\top X_{:, \hat S})^{1/2} \tilde \jac_{\hat S, \hat S}^{(s)}
    &=
    \underbrace{
        \left ( \Id_{n} - \frac{(X_{:, \hat S}^\top X_{:, \hat S})^{1/2} e_{j_s} e_{j_s}^\top (X_{:, \hat S}^\top X_{:, \hat S})^{1/2}}{\norm{X_{:, j_s}}^2} \right )
        }_{A^{(s)}}
     (X_{:, \hat S}^\top X_{:, \hat S})^{1/2} \tilde \jac_{\hat S, \hat S}^{(s-1)}
     \nonumber
    \\
    &- \underbrace{\frac{\sign \hat \beta_{j_s}}{\norm{X_{:, j_s}}^2} (X_{:, \hat S}^\top X_{:, \hat S})^{1/2} }_{B^{(s)}} e_{j_{s}} e_{j_{s}}^\top \nonumber
\end{align}
\begin{equation} \label{eq:var_wlasso}
    (X_{:, \hat S}^\top X_{:, \hat S})^{1/2}
    \tilde \jac_{\hat S, \hat S}^{(\hat s)}
    =\underbrace{A^{(\hat s)}
    \dots A^{(1)} }_{A \in \bbR^{ \hat s \times \hat s}}
    (X_{:, \hat S}^\top X_{:, \hat S})^{1/2}
    \tilde \jac_{\hat S, \hat S}^{(0)}
    + \underbrace{
    A^{(\hat s)} \dots A^{(2)} B^{(1)} e_{j_1} e_{j_1}^\top
    + \dots
    + B^{(\hat s)} e_{j_{\hat s}} e_{j_{\hat s}}^\top }_{D \in \bbR^{\hat s\times \hat s}} \enspace .
\end{equation}
As in the Lasso case, \Cref{eq:var_wlasso} leads to linear convergence once the support is identified for \Cref*{alg:compute_jac_forward_diff_bcd,alg:compute_jac_imp_forward_iter_diff}.
\end{proof}

\section{Block coordinate descent algorithms}

\Cref{alg:compute_jac_forward_diff_bcd} presents the forward iteration scheme which computes iteratively the solution of the Lasso or wLasso jointly with the Jacobian computation.
This is the \textit{naive} way of computing the Jacobian without taking advantage of its sparsity.
Eventually, it requires to differentiate every lines of code \wrt to $\lambda$ and take advantage of the BCD updates for cheap updates on the Jacobian as well.

{\fontsize{4}{4}\selectfont
\begin{algorithm}[h]
\SetKwInOut{Input}{input}
\SetKwInOut{Init}{init}
\SetKwInOut{Parameter}{param}
\caption{\textsc{Forward Iterdiff}
\cite{Deledalle_Vaiter_Fadili_Peyre14,Franceschi_Donini_Frasconi_Pontil17} }
\Input{
    $X \in \bbR^{n \times p},
    y \in \bbR^{n},
    \lambda \in \bbR,
    n_{\text{iter}} \in \bbN$
    }
\tcp{jointly compute coef. \& Jacobian }
$\beta = 0$ \tcp*[l]{ potentially warm started}

$\jac= 0$ \tcp*[l]{potentially warm started}

$r = y - X \beta$

$dr = - X \jac$

    \For{$k = 0,\dots, n_{\text{iter}} - 1$}{
        \For{$j = 0, \dots, p-1$}{
            \tcp*[h]{update the regression coefficients}\\
            $\beta_{\text{old}} = \beta_j$

            $z_j = \beta_j + \frac{1}{\norm{X_{:,j}}^{2}} X_{:, j}^\top r$
            \tcp*[l]{gradient step}
            $\beta_j = \ST(z_j, n e^\lambda / \norm{X_{:,j}}^{2} )$
            \tcp*[l]{proximal step}
            $r \minuseq X_{:, j} (\beta_j - \beta_{\text{old}})$

            \tcp*[h]{update the Jacobian}\\
            \If{Lasso}{
            $\jac_{old} = \jac_j$

            $\jac_j =
             \left | \sign \beta_j \right |
             \left(
                 \jac_j +
                \frac{1}{\norm{X_{:,j}}^{2}} X_{:, j}^\top dr\right)$
            \tcp*[l]{diff. \wrt $\lambda$}
            $\jac_j \minuseq
            \frac{n e^\lambda}{\norm{X_{:,j}}^{2}} \sign \beta_j$
            \tcp*[l]{diff. \wrt $\lambda$}
            $dr_j \minuseq X_{:, j} (\jac_{j} - \jac_{old})$
            %
            }
         \If{wLasso}{
            $\jac_{old} = \jac_{j, :}$

            $\jac_{j, :} =
            | \sign \beta_j |
             \left (\jac_{j, :}
             + \frac{1}{\norm{X_{:,j}}^{2}} X_{:, j}^\top dr \right)$
             \tcp*[l]{diff. \wrt $\lambda_1, \dots,\lambda_p$}
            $\jac_{j, j} \minuseq
            \frac{n e^{\lambda_j}}{\norm{X_{:,j}}^{2}} \sign \beta_j$
            \tcp*[l]{diff. \wrt $\lambda_1, \dots,\lambda_p$}
            $dr \minuseq X_{:, j} (\jac_{j} - \jac_{old})$
            %
        }
        }
    }
\Return{
    $\beta^{ n_{\text{iter}}}, \jac_{(\lambda)}^{ n_{\text{iter}}}$
    }
\label{alg:compute_jac_forward_diff_bcd}
\end{algorithm}
}
%
\Cref{alg:compute_jac_backward_iter_diff} describes the backward iterative differentiation algorithm used for benchmark.
Backward differentiation requires the storage of every updates on $\beta$. As \Cref{fig:intro_influ_niter} shows, this algorithm is not efficient for our case because the function to differentiate $f: \bbR \rightarrow \bbR^{p}$ ( $f: \bbR^{p} \rightarrow \bbR^{p}$, for the wLasso) has a higher dimension output space than the input space.
The storage is also an issue mainly for the wLasso case which makes this algorithm difficult to use in practice in our context.

{\fontsize{4}{4}\selectfont
\begin{algorithm}[h]
\SetKwInOut{Input}{input}
\SetKwInOut{Init}{init}
\SetKwInOut{Parameter}{param}
\caption{\textsc{Backward Iterdiff}  \citep{Domke12}}
\Input{
$X \in \bbR^{n \times p},
y \in \bbR^{n},
\lambda \in \bbR,
n_{\text{iter}} \in \bbN$
}
\tcp{backward computation of $\hbeta$ and $\hat \jac_{(\lambda)}^{\top} \alpha$}

$\beta = 0$ \tcp*[l]{potentially warm started}

\tcp*[l]{compute the regression coefficients and store the iterates}
\For{$k = 0,\dots, n_{\text{iter}} - 1$}{
    \For{$j = 0, \dots, p-1$}{
        $\beta_{\text{old}} = \beta_j$

        $z_j = \beta_j + \frac{1}{\norm{X_{:,j}}^{2}} X_{:, j}^\top r$
        \tcp*[l]{gradient step}

       $\beta_j = \ST(z_j, n e^\lambda / \norm{X_{:,j}}^{2} )$
        \tcp*[l]{proximal step}

        $r \minuseq X_{:, j} (\beta_j - \beta_{\text{old}})$
    }

    }

\tcp{Init. backward differentiation}

$g = 0$ \tcp*[l]{$g$ stores $\hat \jac_{\lambda}^{\top}\alpha$}

\tcp*[l]{compute the Jacobian}
\For{$k = n_{\text{iter}}$ \textbf{down to} $1$}{
    \For{$j = 0, \dots, p-1$}{
        \If{Lasso}{
        $g \minuseq \frac{n e^\lambda}{\norm{X_{:,j}}^{2}} \alpha_j
        \sign \beta_j^{(k)} $

        $\alpha_j \timeseq |\sign \beta_j^{(k)}| $

        $\alpha
        \minuseq \frac{1}{\norm{X_{:,j}}^{2}}\alpha_j X_{:, j}^{\top} X $
        \tcp*[l]{$\bigo(np)$}
        }
        \If{wLasso}{
            $g_j \minuseq \frac{n e^{\lambda_j}}{\norm{X_{:,j}}^{2}} \alpha_j
            \sign \beta_j^{(k)} $

            $\alpha_j \timeseq |\sign \beta_j^{(k)}| $

            $\alpha
            \minuseq \frac{1}{\norm{X_{:,j}}^{2}}\alpha_j X_{:, j}^{\top} X $
        }
    }

}
\Return{
    $\beta^{ n_{\text{iter}}}, g^{(1)}$
    }
\label{alg:compute_jac_backward_iter_diff}
\end{algorithm}
}
%
%
%
\Cref{alg:bcd_lasso} presents the classical BCD iterative scheme for solving the Lasso problem using the composition of a gradient step with the soft-thresholding operator.
{\fontsize{4}{4}\selectfont
\begin{algorithm}[h]
\SetKwInOut{Input}{input}
\SetKwInOut{Init}{init}
\SetKwInOut{Parameter}{param}
\caption{\textsc{BCD for the Lasso}  \citep{Friedman_Hastie_Tibshirani10}}
\Input{
    $X \in \bbR^{n \times p},
    y \in \bbR^{n},
    \lambda \in \bbR,
    \beta^{(0)} \in \bbR^{p},
    n_{\text{iter}} \in \bbN$
    }

$\beta = \beta^{(0)}$ \tcp*[l]{warm start}

\For{$k = 0,\dots, n_{\text{iter}} - 1$}{
    \For{$j = 0, \dots, p-1$}{
        $\beta_{\text{old}} = \beta_j$

        $z_j = \beta_j + \frac{1}{\norm{X_{:,j}}^{2}} X_{:, j}^\top r$
        \tcp*[l]{gradient step}

       $\beta_j = \ST(z_j, n e^\lambda / \norm{X_{:,j}}^{2} )$
        \tcp*[l]{proximal step}

        $r \minuseq X_{:, j} (\beta_j - \beta_{\text{old}})$
        }
    }
\Return{
    $\beta^{ n_{\text{iter}}}$
    }
\label{alg:bcd_lasso}
\end{algorithm}
}



\clearpage
\section{Derivations for MCP}
\label{app:sec_mcp}

Let us remind the definition of the Minimax Concave Penalty (MCP) estimator introduced by \citet{Zhang10}, also analyzed under the name CELE0 by \citet{Soubies_Blanc-FeraudAubert16}.
First of all, for any $t\in\bbR$:
\begin{align}
    p_{\lambda, \gamma}^{\text{MCP}}(t)=
    \begin{cases}
        {\lambda |t|-\frac{t^{2}}{2 \gamma},} & {\text { if } |t| \leq \gamma \lambda} \\
        {\frac{1}{2} \gamma \lambda^2,}
        & {\text { if } |t| > \gamma \lambda}
        \enspace.
    \end{cases}
\end{align}
The proximity operator of $p_{\lambda, \gamma}$ for parameters $\lambda>0$ and $\gamma>1$ is defined as follow (see \citealt[Sec. 2.1]{Breheny_Huang11}):
\begin{align}
    \prox^{\text{MCP}}_{\lambda, \gamma}(t)=
    \begin{cases}
        {\frac{\ST(t, \lambda)}{1 - \frac{1}{\gamma}}} & {\text { if } |t| \leq \gamma \lambda} \\
        {t} & {\text { if } |t| > \gamma \lambda}
        \enspace.
    \end{cases}
\end{align}
For ourselves we choose as for the Lasso an exponential parametrization of the coefficients, for $\lambda \in \bbR$ and $\gamma > 0$:
\begin{align}
    \hat \beta^{(\lambda, \gamma)}(y)
    \eqdef \argmin_{\beta \in \bbR^p}
    \frac{1}{2n} \normin{y - X \beta}_2^2
    + \sum_{j=1}^{p} p_{e^\lambda, e^\gamma}^{\text{MCP}}\left(\left|\beta_{j}\right|\right)
    \enspace.
\end{align}
%

%

\paragraph{Update rule for Coordinate Descent}

Below, we provide equation to update the coefficient in the coordinate descent algorithm of the MCP:
\begin{align}
\beta_j
    & \leftarrow
    \argmin_{\beta_j \in \bbR}
    \frac{1}{2n} \normin{y - \beta_j X_{:, j} - \sum_{j'\neq j} \beta_{j'} X_{:, j'}}_2^2
    +
    \sum_{j'\neq j}^p
    p_{e^\lambda, e^\gamma}^{\text{MCP}}(\beta_{j'})
    +
    p_{e^\lambda, e^\gamma}^{\text{MCP}}(\beta_{j})
    \nonumber\\
    & =
    \argmin_{\beta_j \in \bbR}
    \frac{1}{2n} \normin{y - \beta_j X_{:, j} - \sum_{j'\neq j} \beta_{j'} X_{:, j'}}_2^2
    +p_{e^\lambda, e^\gamma}^{\text{MCP}}(\beta_{j})
    \nonumber\\
    & =
    \argmin_{\beta_j \in \bbR}
    \norm{X_{:, j}}_2^2
    \left(
        \frac{1}{2n}
        \left[
        \beta_j
        -
        \frac{1}{\norm{X_{:, j}}_2^2}
            \left \langle
            y-\sum_{j'\neq j} \beta_{j'} X_{:, j'}, X_{:, j}
            \right \rangle
        \right]^2
        +
        \frac{1}{\norm{X_{:, j}}_2^2}
        p_{e^\lambda, e^\gamma}^{\text{MCP}}
    (\beta_j)
    \right)\nonumber\\
    & =
    \argmin_{\beta_j \in \bbR}
    \left(
        \frac{1}{2n}
        \left[
        \beta_j
        -
        \frac{1}{\norm{X_{:, j}}_2^2}
            \left \langle
            y-\sum_{j'\neq j} \beta_{j'} X_{:, j'}, X_{:, j}
            \right \rangle
        \right]^2
        +
        \frac{1}{\norm{X_{:, j}}_2^2}
        p_{e^\lambda, e^\gamma}^{\text{MCP}}(\beta_j)
    \right) \nonumber\\
    & =  \argmin_{\beta_j \in \bbR}
    \left(
        \frac{1}{2L_j}
        \left[
        \beta_j
        -
        \frac{1}{\norm{X_{:, j}}_2^2}
            \left \langle
            y-\sum_{j'\neq j} \beta_{j'} X_{:, j'}, X_{:, j}
            \right \rangle
        \right]^2
    +
    p_{e^\lambda, e^\gamma}^{\text{MCP}}(\beta_{j})
    \right),
    \text{with $L_j \eqdef \frac{n}{\norm{X_{:, j}}_2^2}$}
    \nonumber\\
    & = \prox_{e^\lambda / L_j, e^\gamma L_j}^{\text{MCP}}
    \left(
        \beta_j-\frac{1}{\norm{X_{:, j}^2}} X_{:, j}^\top(X\beta-y), \lambda
    \right)
    \enspace .
\end{align}

One can write the following fixed point equation satisfied by the estimator $\ \hat \beta$, with $L_j = \norm{X_{:, j}}^2 / n$:
\begin{align}\label{eq:fixed_point}
    \hbeta_j &=
    \prox_{e^\lambda / L_j, e^\gamma L_j}^{\text{MCP}}
    \left(
        \left\langle
         y-\sum_{k\neq j} \hbeta_k X_{:, k}, \frac{X_{:, j}}{\norm{X_{:, j}}^2}
        \right\rangle
    \right)\nonumber\\
    & =
    \prox_{e^\lambda / L_j, e^\gamma L_j}^{\text{MCP}}
    \left(
        \hbeta_j
        -\frac{1}{\norm{X_{:, j}}^2}
        X_{:, j}^\top\left(
         X\hbeta-y
        \right)
    \right) \enspace .
\end{align}
Since the MCP penalty is non-convex, the estimator may not be continuous \wrt hyperparameters and gradient based hyperparameter optimization may not be theoretically justified.
However we can differentiate the fixed point equation \Cref{eq:fixed_point} almost everywhere:
\begin{align}
    \hat \jac_{j} =&
        \left(
            \hat\jac_{j} - \frac{1}{\norm{X_{:j}}^{2}_{2}}X^{\top}_{:j}X\hat\jac
        \right)
    \cdot
    \frac{
        \partial
        \prox_{e^\lambda / L_j, e^\gamma L_j}^{\text{MCP}}
    }{\partial t}
    \left(\hbeta_j-\frac{1}{\norm{X_{:, j}^2}} X_{:, j}^\top(X\beta-y)\right) \nonumber\\
    &+
    \frac{e^\lambda}{L_j}
    \frac{
        \partial \prox_{e^\lambda / L_j, e^\gamma L_j}^{\text{MCP}}
    }{\partial \lambda}\left(\hbeta_j-\frac{1}{\norm{X_{:, j}^2}} X_{:, j}^\top(X\beta-y)\right) \nonumber\\
    &+
    e^\gamma L_j
    \frac{
        \partial \prox_{e^\lambda / L_j, e^\gamma L_j}^{\text{MCP}}
    }{\partial \gamma}\left(\hbeta_j-\frac{1}{\norm{X_{:, j}^2}} X_{:, j}^\top(X\beta-y)\right)
    \enspace.
\end{align}
where
\begin{align}
        \frac{
            \partial
            \prox_{\lambda, \gamma}^{\text{MCP}}
        }{\partial t}(t)
    &=
    \begin{cases}
        \frac{|\sign t |}{1- \frac{1}{{\gamma}}} , \quad  &\text{if } |t| \leq \lambda  \gamma\\
        1, \quad  &\text{otherwise}
    \end{cases} \enspace, \\
    \frac{
        \partial
        \prox_{\lambda, \gamma}^{\text{MCP}}
    }{\partial \lambda}(t)
    &=
    \begin{cases}
        0, \quad
        &\text{ if } |t| \leq \lambda\\
        - \frac{\sign t}{1- \frac{1}{{\gamma}}}, \quad
        &\text{ if }  \lambda\leq |t| \leq \lambda \gamma\\
        0, \quad
        &\text{ if } |t| > \lambda \gamma
        \end{cases}
        \enspace, \\
    \frac{
        \partial
        \prox_{\lambda, \gamma}^{\text{MCP}}
    }{\partial \gamma}(t)
    &=
    \begin{cases}
            - \frac{\ST\left(
            t, \lambda
            \right)}{
            \left(
                \gamma - 1
            \right)^2} \quad
        &\text{ if } |t| \leq \lambda \gamma \\
        0
        &\text{ if } |t| > \lambda \gamma\\
    \end{cases}
     \enspace.
\end{align}
%
%
\begin{figure}[tb]
    \begin{center}
  \begin{subfigure}[b]{0.45\textwidth}
      \centering
      \includegraphics[width=\figsize\linewidth]{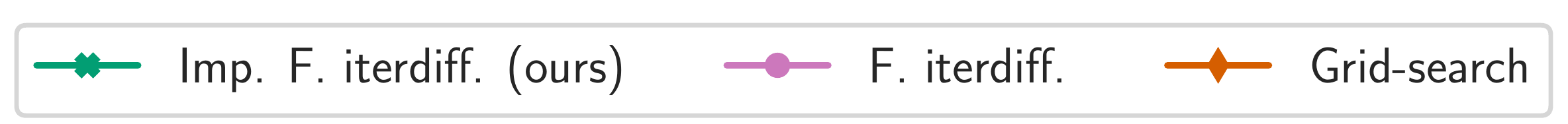}
      \includegraphics[width=\figsize\linewidth]{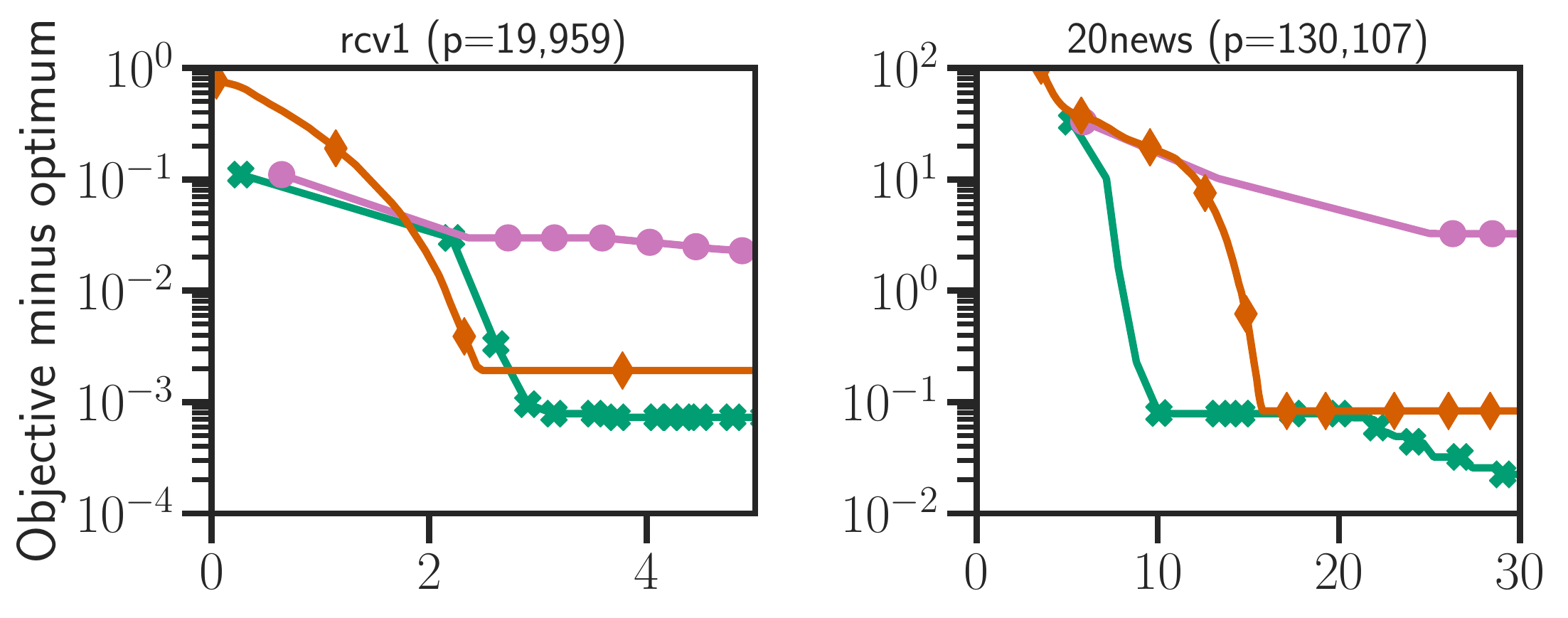}
  \end{subfigure}%

  \begin{subfigure}[b]{0.45\textwidth}
        \includegraphics[width=\figsize\linewidth]{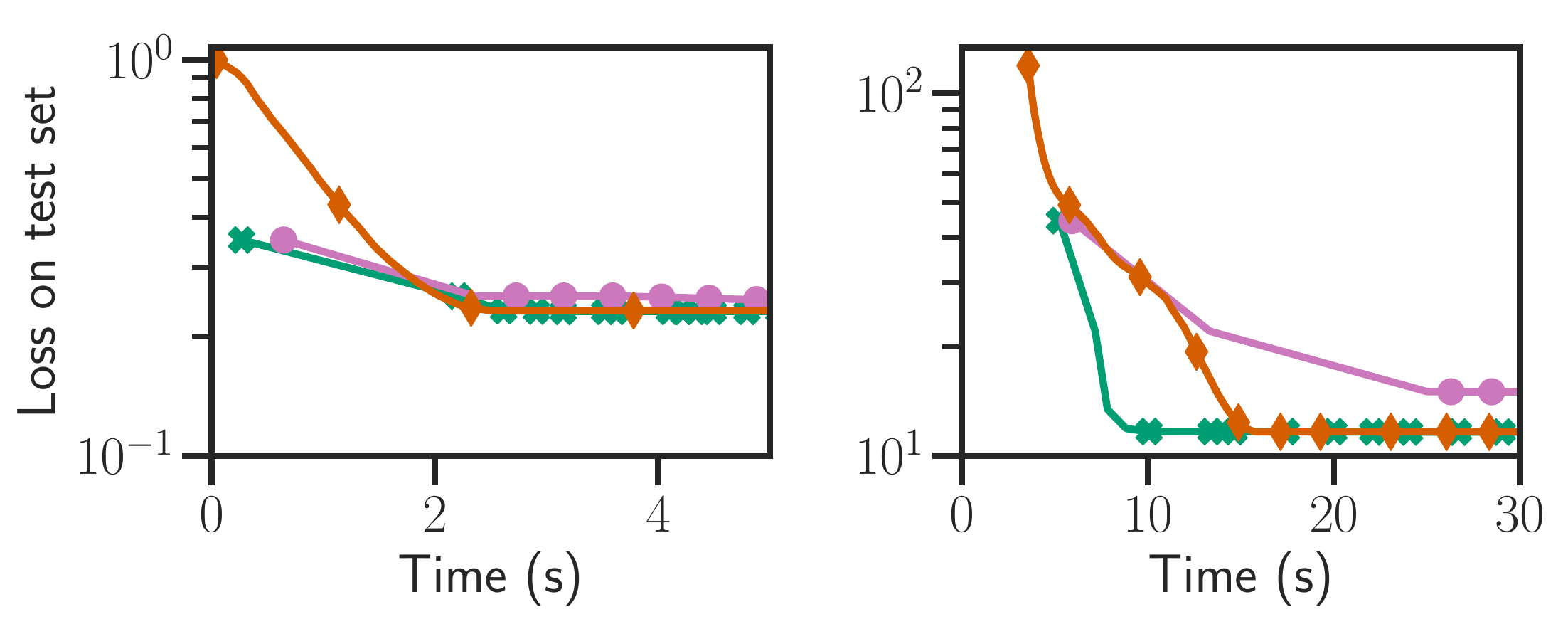}
    \end{subfigure}%
\end{center}

    \caption{\textbf{Computation time for the HO of the MCP on real data} Distance to ``optimum'' (top) and performance (bottom) on the test set for the MCP.}
  \label{fig:mcp_train_test_perf}

\end{figure}
Contrary to other methods, HO based algorithms do not scale exponentially in the number of hyperparameters.
Here we propose experiments on the held-out loss with the MCP estimator \citep{Zhang10}, which has 2 hyperparameters $\lambda$ and $\gamma$.
Our algorithm can generalize to such non-smooth proximity-based estimator.

\textit{Comments on \Cref{fig:mcp_train_test_perf}
(MCP,
held-out criterion).
}
\Cref{fig:mcp_train_test_perf} (top) shows the convergence of the optimum on 2 datasets (\emph{rcv1} and \emph{20news}) for the MCP estimator.
As before \ourfull outperforms \forwardfull illustrating \Cref{prop:convergence_iterdiff} and \Cref{tab:summary_costs}.

\clearpage

\section{Datasets and implementation details}
\label{app:dataset}

The code used to produce all the figures as well as the implementation details can be found in the supplementary material in the $forward\_implicit/expes$ folder.
In particular in all experiments, for our algorithm, \ourfull, the size of the loop computing the Jacobian is fixed: $\text{n}\_\text{iter}\_\text{jac} = 100$.
Reminding that the goal is to compute the gradient:
\begin{equation}
    \hat \jac_{(\lambda)}^\top \nabla \mathcal{C} \left ( \hat \beta^{(\lambda)} \right ) \enspace,
\end{equation}
we break the loop if
\begin{equation}
    \normin{(\jac^{(k+1)} - \jac^{(k)})\nabla \mathcal{C} ( \hat \beta^{(\lambda)} ) }
    \leq
    \normin{\nabla \mathcal{C} (\hat \beta^{(\lambda)} )} \times \epsilon^{\text{jac}} \enspace,
\end{equation}
with $\epsilon^{\text{jac}} = 10^{-3}$.
All methods benefit from warm start.

\subsection{Details on \Cref{fig:intro_influ_niter}}
%
\Cref{fig:intro_influ_niter} is done using synthetic data.
As described in \Cref{sub:sugar}, $X \in \bbR^{n \times p}$ is a Toeplitz correlated matrix, with correlation coefficient
$\rho = 0.9$, $(n, p) = (1000, 2000)$.
$\beta \in \bbR^p$ is chosen with $5$ non-zero coefficients chosen at random.
Then $y \in \bbR^n$ is chosen to be equal to $X \beta$ contaminated by some \iid random Gaussian noise, we chose $\text{SNR}=3$.
For \Cref{fig:intro_influ_niter} all the implementation details can be found in the joint code in the
$forward\_implicit/examples/plot\_time\_to\_compute\_single\_gradient.py$
file.
\Cref{fig:intro_influ_niter} shows the time of computation of one gradient and the distance to ''optimum''.
For this figure we evaluated the gradient in  $\lambda = \lambda_{\text{max}} - \ln(10)$.
The ''optimum'' is the gradient obtained using the implicit differentiation method.

\subsection{Details on \Cref{fig:Lasso_train_test_perf}}
%
Let us first begin by a description of all the datasets and where they can be downloaded.

\paragraph{\emph{rcv1}.}
The \emph{rcv1} dataset can be downloaded here:
\url{https://www.csie.ntu.edu.tw/~cjlin/libsvmtools/datasets/multilabel.html#rcv1v2%20(topics;%20subsets)}.
The dataset contains $n=20,242$ samples and  $p=19,959$ features.

\paragraph{\emph{20news}.}
The \emph{20news} dataset can be downloaded here
\url{https://www.csie.ntu.edu.tw/~cjlin/libsvmtools/datasets/multiclass.html#news20}.
The dataset contains $n= 11,314$
samples and  $p=130,107$ features.

\paragraph{\emph{finance}.}
The \emph{finance} (\emph{E2006-log1p} on libsvm) dataset can be downloaded here:
\url{https://www.csie.ntu.edu.tw/~cjlin/libsvmtools/datasets/regression.html#E2006-log1p}.
The dataset contains $n= 16,087$
samples and  $p=1,668,737$ features.

All the implementation details can be found in the code:
$forward\_implicit/expes/main\_lasso\_pred.py$.

\subsection{Details on \Cref{fig:lasso_estimation}}
%
\Cref{fig:lasso_estimation} was performed using simulated data. The matrix $X\in \bbR^{n\times p}$ was obtained by simulated $n\times p$ \iid Gaussian variables $\mathcal{N}(0,1)$.
The number of rows was fixed at $n=100$ and we changed the number of columns $p$ from 200 to 10,000 on a linear grid of size 10.
Then , we generated $\beta^{*}$ with 5 coefficients equal to 1 and the rest equals to 0.
The vector $y$ is equal to $X\beta^{*}$ contaminated by some \iid random Gaussian noise controlled by a SNR value of 3.
We performed 50 repetitions for each value of $p$ and computed the average MSE on these repetitions.
The initial value for the line-search algorithm was set at $\lambda_{\text{max}} + \ln(0.7)$ and the number of iterations for the Jacobian at $500$ for the whole experiment.
All the implementation details can be found in the code : $forward\_implicit/expes/main\_lasso\_est.py$.

\subsection{Details on \Cref{fig:Lasso_vs_wLasso}}
%
\Cref{fig:Lasso_vs_wLasso} was performed using the same simulating process as described above only this time we performed only 25 repetitions for each value of p.
We had to deal with the fact that \Cref{eq:bilevel_opt} is not convex for the weighted Lasso which means that our line-search algorithm could get stuck in local minima.
In order to alleviate this problem, we introduced \Cref{eq:l2_bilevel_opt} to obtain an initial point for the line-search algorithm.
We chose the regularization term to be constant and equals to $C(\beta^{(\lambda_{\text{max}})})/10$.
We used a time treshold of $500$ seconds which was hit only by the forward differentiation algorithm for the wLasso.
The details about this experiment can be found in the code : $forward\_implicit/expes/main\_wLasso.py$.

\clearpage

\section{Supplementary experiments}
\label{app:expe_sppl}

\subsection{Experiments with a non-unique solution to the inner problem}
\label{app:non_unique}

We recall here that the bi-level optimization \Cref{eq:bilevel_opt} is solved using gradient descent.
We recall also that gradient descent may not converge toward a global minima since the optimized function  $\lambda \mapsto \cL(\lambda)$ may not be convex.
It may be even worse: if the inner optimization problem has not a unique solution, the function $\lambda \mapsto \cL(\lambda)$ may not be continuous.
However our algorithm can still be applied to compute the hypergradient.
\Cref{fig:intro_influ_niter_non_unique} shows the time to compute a single (hyper)gradient when the solution to the inner problem is not unique.

As proved for instance in \citet[Lemma 3 and 4]{Tibshirani13},
the set of parameters where the Lasso solution is not unique is typically $\emptyset$ or a set whose Lebesgue measure is zero. Moreover, there exist settings such that the solution path (as a multivalued mapping) could be non-continuous, which leaves only non-gradient based methods available. Thus, we decided to not investigate the theory in such pathological settings. The authors are not aware of a classical dataset where non-uniqueness arises.
Nevertheless, in the case where there exists $\lambda$ such that the solution set is not reduced to a singleton, our proposed algorithm can still be applied to any solution without theoretical guarantees.

\textbf{Experimental setting for non-uniqueness.}
For completeness, we run our methods on the following toy example \citet{Tibshirani13}: we consider a design $X$ such that $n = 100$, $p = 10000$ and
$X_1, X_2, X_3$
are generated iid following a standard normal distribution,
$X_4 = (X_2 + X_3)/2$ and $X_5, \dots, X_{p}$ are generated \iid following a standard normal distribution, then orthogonalized $X_5, \dots, X_p$ \wrt $\Span(X_1, \dots, X_4)$.
We let $y = -X1 + X2 + X3$.
We let the reader observe that when $\lambda = 1 / n$, the solution set is not reduced to a singleton. In this case, similar conclusions are observed in \Cref{fig:intro_influ_niter_non_unique} as for other datasets: Imp. F. Iterdiff (ours) still outperforms its competitors.
%
%
\begin{figure}[H]
    \includegraphics[width=0.8\linewidth]{legend_intro_influ_niter}
    \centering
      \includegraphics[width=0.8\linewidth]{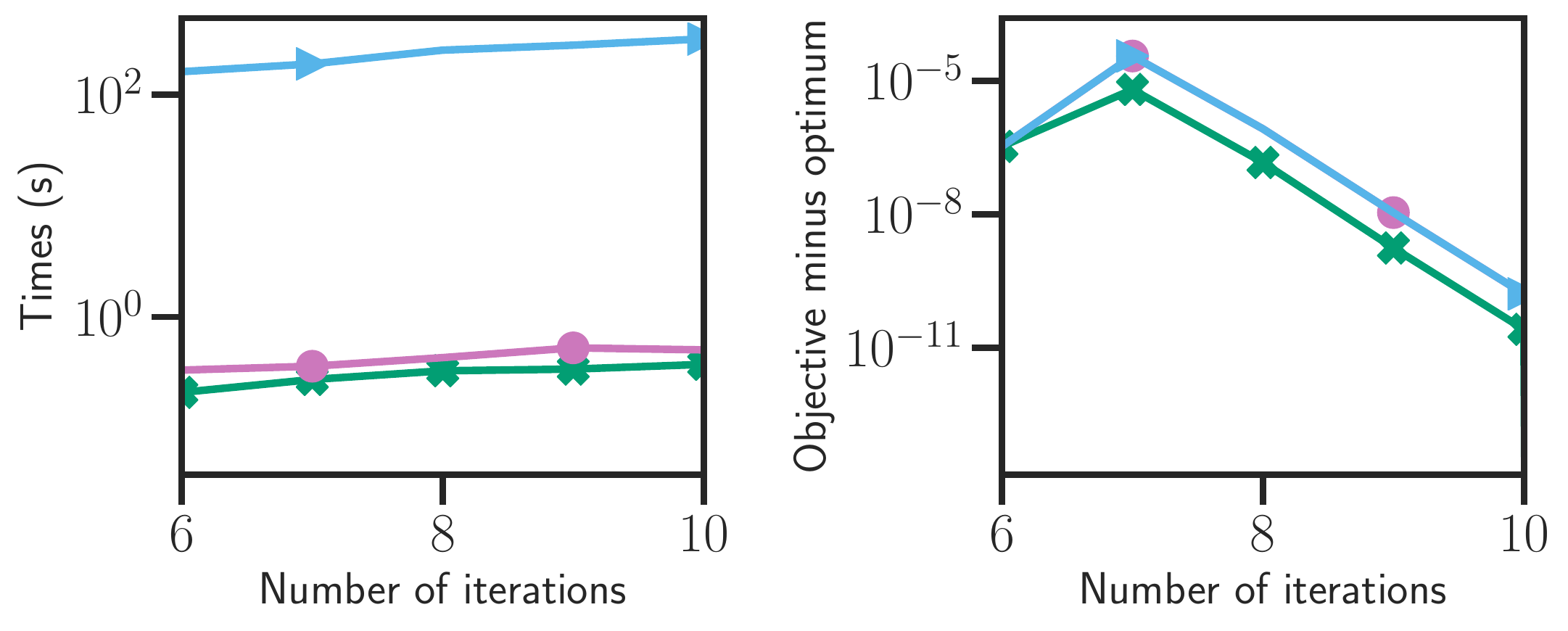}
  \caption{\textbf{Time to compute a single gradient with non-unique solution} (Synthetic data, Lasso, $n, p = 1000, 10000$).
  Influence on the number of iterations of BCD (in the inner optimization problem of \Cref{eq:bilevel_opt}) on the computation time (left) and the distance to ``optimum'' of the gradient $ \nabla_\lambda \cL(\lambda) $(right) for the Lasso estimator.
  }
  \label{fig:intro_influ_niter_non_unique}
\end{figure}
%

\end{document}